\documentclass[twoside]{article}

\usepackage[accepted]{aistats2021}
\usepackage[round]{natbib}
\usepackage{graphicx}

\usepackage{amsmath}
\usepackage{amssymb}
\usepackage{mathtools}
\usepackage{pifont}
\usepackage[dvipsnames]{xcolor}
\usepackage{amsthm}
\usepackage{comment}
\usepackage{hyperref}

\newtheorem{theorem}{Theorem}
\newcommand{\x}{\boldsymbol{x}}
\newcommand{\y}{\boldsymbol{y}}
\newcommand{\z}{\boldsymbol{z}}

\newcommand{\s}{\boldsymbol{s}}

\newcommand{\f}{\boldsymbol{f}}
\newcommand{\g}{\boldsymbol{g}}

\newcommand{\Y}{\mathcal{Y}}

\newcommand{\X}{\mathcal{X}}

\newcommand{\Z}{\mathcal{Z}}

\newcommand{\R}{\mathbb{R}}
\newcommand{\cov}{\mathrm{cov}}
\renewcommand{\vec}{\mathrm{vec}}
\renewcommand{\u}{\boldsymbol{u}}
\newcommand{\cmark}{\color{ForestGreen}\ding{51}}%
\newcommand{\xmark}{\color{RubineRed}\ding{55}}
\DeclarePairedDelimiterX{\infdivx}[2]{(}{)}{%
  #1\;\delimsize\|\;#2%
}

\begin{document}

%

%
\runningauthor{Ramchandran, Tikhonov, Kujanp\"a\"a, Koskinen, L\"ahdesm\"aki}

\twocolumn[

\aistatstitle{Longitudinal Variational Autoencoder}

\aistatsauthor{Siddharth Ramchandran\textsuperscript{1} \hspace{1.2cm} Gleb Tikhonov\textsuperscript{1} \hspace{1.2cm} Kalle Kujanp\"a\"a\textsuperscript{1} \\ \textbf{Miika Koskinen\textsuperscript{2,3}} \hspace{1.2cm} \textbf{Harri L\"ahdesm\"aki\textsuperscript{1}}}

\aistatsaddress{\\\textsuperscript{1}Department of Computer Science, Aalto University, Finland\\\textsuperscript{2}HUS Helsinki University Hospital, Finland\\ \textsuperscript{3}Faculty of Medicine, University of Helsinki, Finland\\ \href{mailto:siddharth.ramchandran@aalto.fi}{\nolinkurl{siddharth.ramchandran@aalto.fi}}}]

\begin{abstract}
Longitudinal datasets measured repeatedly over time from individual subjects, arise in many biomedical, psychological, social, and other studies. 
A common approach to analyse high-dimensional data that contains missing values is to learn a low-dimensional representation using variational autoencoders (VAEs).
However, standard VAEs assume that the learnt representations are i.i.d., and fail to capture the correlations between the data samples. We propose the Longitudinal VAE (L-VAE), that uses a multi-output additive Gaussian process (GP) prior to extend the VAE's capability to learn structured low-dimensional representations imposed by auxiliary covariate information, and  derive a new KL divergence upper bound for such GPs. Our approach can simultaneously accommodate both time-varying shared and random effects, produce structured low-dimensional representations, disentangle effects of individual covariates or their interactions, and achieve highly accurate predictive performance. We compare our model against previous methods on synthetic as well as clinical datasets, and demonstrate the state-of-the-art performance in data imputation, reconstruction, and long-term prediction tasks.
\end{abstract}

\section{Introduction}
Longitudinal datasets naturally arise in a wide variety of fields and applications, such as biomedicine, sociology, psychology, and many others. Such datasets include, for example, healthcare records, social media behaviour, consumer behaviour, etc., all collected repeatedly over time for each individual. Most longitudinal datasets contain both dependent and independent variables. For example, in biomedical data, dependent variables can comprise of lab tests and other measurements 
of the patient, 
whereas independent variables contain  auxiliary descriptors of the patient, such as age, sex, time to disease event, etc. Analysing longitudinal datasets collected in such studies is challenging as they often involve time-varying covariates, high-dimensional correlated measurements, and missing values. 

While non-linear, high-dimensional generative models are capable of learning complex data distributions, the statistical inference for such models is generally highly non-trivial. Auto-Encoding Variational Bayes (AEVB) \citep{kingma2013auto} is a powerful deep learning technique for efficient inference of latent variable models. The variational autoencoder (VAE) \citep{kingma2013auto, rezende2014stochastic}, the most popular exemplification of AEVB, learns a low-dimensional latent code of the dataset using two complementary deep neural networks (DNNs) to encode the high-dimensional data and decode the latent distribution, respectively. However, VAEs usually ignore the possible correlations (e.g.\ temporal correlations) between the learnt latent embeddings.
\paragraph{Related work} Numerous extensions to achieve correlations in the latent space, model temporal data, and enhance the expressiveness of posterior distributions have been proposed for VAEs. \citet{kulkarni2015deep} had proposed to group samples with specific properties in mini-batches to induce structure on the latent space. The conditional VAE (CVAE) introduced in \citep{sohn2015learning} incorporated the auxiliary covariate information directly in the inference and generative networks. However, CVAE fails to model the subject-specific temporal structure and does not explicitly constrain the latent space to achieve a low dimensional representation that evolves smoothly in time. \citet{casale2018gaussian} proposed the GPPVAE to incorporate view and object information in a Gaussian process (GP) prior, to model the VAE’s latent space structure. GPPVAE can account for temporal covariances between samples, but its ability to model subject-specific temporal structure is limited by the restrictive nature of the view-object GP product kernel. This compromises the applicability of GPPVAE in longitudinal study designs. Moreover, GPPVAE's pseudo-minibatch stochastic gradient descent (SGD) training scheme lacks the ability to scale to large data, as each training step requires a pass over the full data (epoch). \citet{fortuin2019multivariate} built upon the idea of using a latent GP in VAEs, and proposed the GP-VAE that assumes an independent GP prior on each subject's time-series. Though GP-VAE is especially designed for time-series data, it can neither capture shared temporal structure across all data points nor make use of available auxiliary covariate information other than time. 

Limitations of the expressiveness of posterior approximations in VAEs has been addressed by using normalising flows (NF) \citep{rezende2015variational} implemented with RealNVPs \citep{dinh2016density}, continuous-time NFs \citep{chen2017continuous}, inverse autoregressive flows \citep{kingma2016improved}, and importance sampling \citep{muller2018neural}. Methods have also been proposed to handle the disentanglement of dimensions in the latent space \citep{ainsworth2018disentangled,ainsworth2018oi, higgins2017beta} and improve latent representations \citep{alemi2018fixing, zhao2019infovae}. All these methods, however, assume independence across samples.

From the deep neural networks perspective, recurrent architectures (RNN) have been found to be particularly well-suited for temporal data analysis \citep{pearlmutter1989learning, giles1994dynamic}, including multi-outcome modelling problems.  For example, \citet{chung2015recurrent} proposed the variational RNN (VRNN) which extends the VAE into a recurrent framework for modelling highly structured sequential data. The VRNN models the temporal dependencies between the latent random variables across time steps. However, \citet{chung2015recurrent} do not propose a way to handle and impute missing values. Also, VRNNs neither makes use of auxiliary covariates nor takes into account differing time steps. BRITS \citep{cao2018brits} makes use of bi-directional LSTM-type RNNs \citep{schuster1997bidirectional}, and can efficiently impute missing values while accounting for the irregularities in the sampling times. However, BRITS is not a generative model which can limit the applicability of the trained model. Moreover, it is not straightforward to incorporate auxiliary information. 
Generative adversarial networks (GANs) can also be used for time-series data imputation and modelling \citep{goodfellow2016deep, guo2019data, luo2018multivariate}. GRUI-GAN \citep{luo2018multivariate} is an RNN-based method that makes use of adversarial training. This recurrent model suffers from similar pitfalls as BRITS with the added complexity of adversarial training. Table \ref{table:comparison} contrasts the features of our proposed model to the key related methods.    

\begin{figure}[t]
\begin{center}
\includegraphics[width=\linewidth]{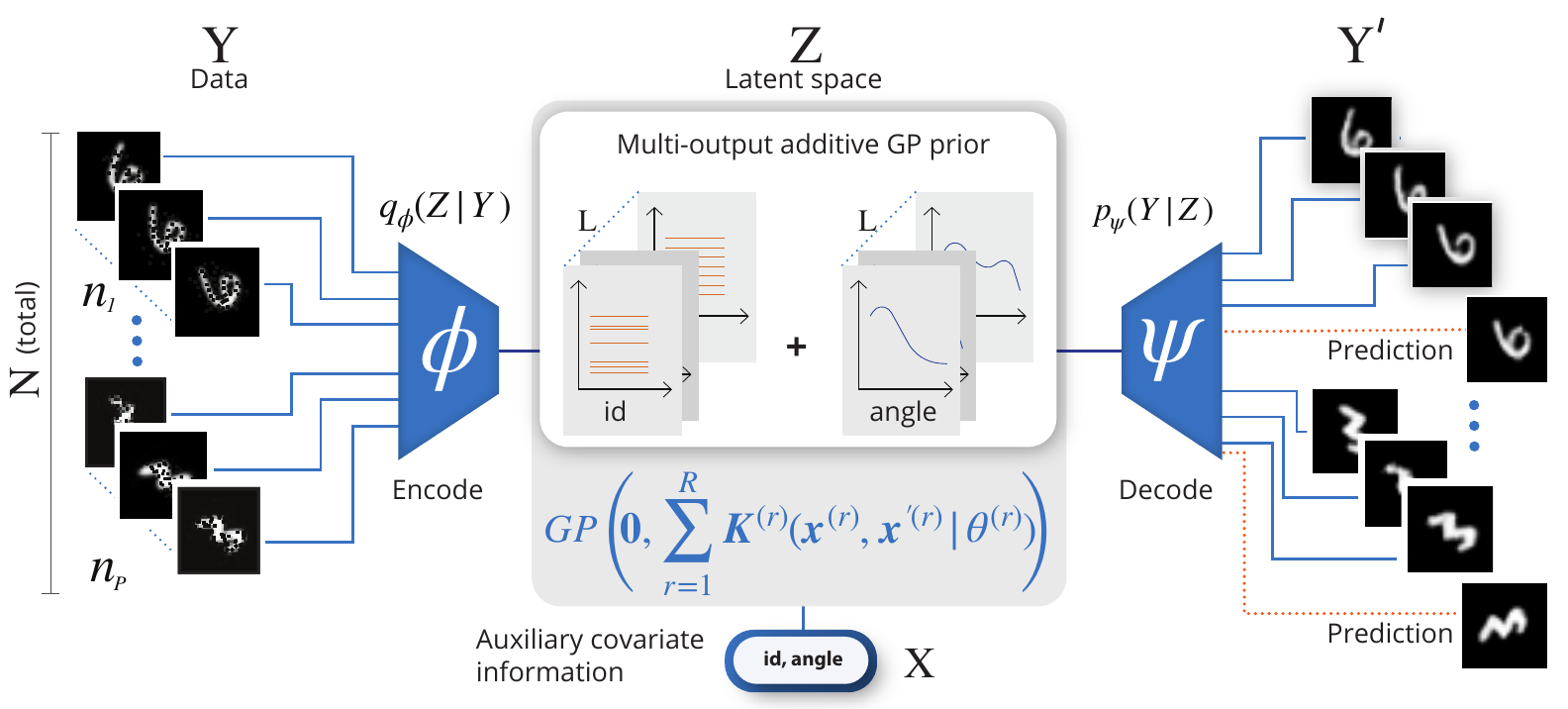}
\caption{L-VAE overview.}
\label{fig:intro}
\end{center}
\end{figure}
\begin{table*}[t]
\caption{Comparison of related methods.} \label{table:comparison}
\centering
\resizebox{\textwidth}{!}{%
\begin{tabular}{lccccccr}
\hline \\
\textbf{Models} & \begin{tabular}[c]{@{}c@{}}\textbf{Shared temporal}\\ \textbf{structure}\end{tabular} & \begin{tabular}[c]{@{}c@{}}\textbf{Individual}\\ \textbf{temporal structure}\end{tabular} & \textbf{Other covariates} & \textbf{Minibatching} & \begin{tabular}[c]{@{}c@{}}\textbf{Temporal}\\ \textbf{irregularities}\end{tabular} & \textbf{Generative} & \textbf{Reference} \\
\hline
VAE    &        \xmark     &       \xmark     &    \xmark        &     \cmark     &       \xmark      &   \cmark    &     \citet{kingma2013auto}\\
CVAE   &        \cmark     &       \xmark     &    \cmark        &     \cmark     &       \xmark      &   \cmark    &     \citet{sohn2015learning}\\
GPPVAE &        \cmark     &      Limited     &    Limited       &      Pseudo    &       \xmark      &   \cmark    &     \citet{casale2018gaussian}\\
GP-VAE &        \xmark     &       \cmark      &    \xmark        &      \cmark    &       \xmark      &   \cmark    &     \citet{fortuin2019multivariate}\\
VRNN  &        \cmark     &       \xmark      &    \xmark        &      \cmark    &       \xmark      &   \cmark    &     \citet{chung2015recurrent}\\
BRITS  &        \cmark     &       \xmark      &    \xmark        &      \cmark    &       \cmark      &   \xmark    &     \citet{cao2018brits}\\
GRUI-GAN &      \cmark     &       \xmark     &    \xmark        &      \cmark    &       \cmark      &   \cmark    &     \citet{luo2018multivariate}\\
\hline
L-VAE  &        \cmark     &       \cmark     &    \cmark        &      \cmark    &       \cmark      &   \cmark    &     Our work\\
\hline
\end{tabular}}
\end{table*}
\paragraph{Contributions} In this paper, we propose a novel deep generative model that extends the capabilities of a VAE with a multi-output additive GP prior over the latent encodings domain, that models the correlation structure between the samples w.r.t.\ auxiliary information. Our L-VAE model is conceptualised in Fig.~\ref{fig:intro}. Our model probabilistically encodes the longitudinal measurements with missing values (missing completely at random) onto a low-dimensional latent space with no missing values. The structured, low-dimensional latent dynamics are modelled using a multi-output additive GP that utilises the auxiliary covariates, followed by decoding back to the data domain. Such a GP prior introduces computational challenges for which we derive a novel divergence bound that leverages the commonly used inducing point formalism \citep{quinonero2005unifying, titsias2009variational} for efficient model inference. Our contributions can be summarised as follows:
\begin{itemize}
    \item We introduce a VAE for longitudinal data with auxiliary covariate information, that can model the structured latent space dynamics with a multi-output additive GP prior.
    \item We derive an efficient mini-batch based GP inference scheme by exploiting the natural structural properties of the additive GP covariance functions (CF) used for longitudinal modelling.
    \item We compare L-VAE's performance against competing methods on several datasets and report the state-of-the-art performance in imputation, reconstruction, and long-term prediction.
\end{itemize}
The source code is available at \url{https://github.com/SidRama/Longitudinal-VAE} .

\section{Methods}
\paragraph{Problem setting} Let $D$ be the dimensionality of the observed data, $P$ be the number of unique instances (e.g.\ unique patients, unique handwritten digit styles), $n_p$ be the number of longitudinal samples from instance $p$, and $N = \sum_{p=1}^{P}n_p$ be the total number of samples (e.g.\ the total number of measurements for all patients, total number of images). Longitudinal samples for instance $p$ are denoted as $Y_p = [\y_{1}^{p}, \ldots, \y_{n_p}^{p}]^T$, where each sample $\y_t^{p} \in \Y$. 
In this work, we assume $\Y = \R^D$. We represent the auxiliary covariate information for the instance $p$ as $X_p = [\x_1^{p},\ldots,\x_{n_p}^{p} ]^T$, where covariates for each sample $\x_t^{p} \in \X = \X_1 \times \ldots \times \X_Q$, $\X_q$ is the domain of the $q^{\mathrm{th}}$ covariate, and $Q$ is the number of covariates. 
For example in Electronic Health Record (EHR) data, covariate information can include patient information such as age, sex, weight, time since remission, etc. Collectively, instance-specific samples and covariates form the full longitudinal data matrix $Y = [Y_1^T, \ldots, Y_P^T]^T =  [\y_1,\ldots,\y_N]^T$ and covariate matrix $X = [X_1^T, \ldots, X_P^T]^T = [\x_1,\ldots,\x_N]^T$, respectively. 
We denote the low-dimensional latent space 
as $\Z = \R^L$ and a latent embedding for all $N$ samples as $Z = [\z_1,\ldots,\z_N ]^T \in \mathbb{R}^{N \times L}$. We also index $Z$ across $L$ dimensions as $Z = [\bar{\z}_{1},\ldots,\bar{\z}_{L}]$, where $\bar{\z}_{l} = [z_{1l},\ldots,z_{Nl}]^T$ is a vector that contains the $l^\mathrm{th}$ dimension of the latent embedding for all $N$ samples.

\paragraph{Auto-encoding variational Bayes} Consider the joint generative model $p_\omega(\y,\z) = p_\psi(\y|\z)p_\theta(\z)$ parameterised by $\omega=\{\psi,\theta\}$, and assume we are interested to infer the latent variable $\z$ given $\y$. The posterior distribution $p_\omega(\z|\y) = p_\psi(\y|\z)p_\theta(\z)/p_\omega(\y)$ is generally intractable due to the marginalisation over the latent space $p_\omega(\y) = \int p_\psi(\y|\z)p_\theta(\z)dz$. 
Auto-Encoding Variational Bayes (AEVB) \citep{kingma2013auto} is a powerful deep learning technique for latent variable models that factorise across samples as $p_\omega(Y,Z) = p_\psi(Y|Z)p_\theta(Z) = \prod_{n=1}^N p_\psi(\y_n|\z_n)p_\theta(\z_n)$.
AEVB introduces an inference model (also called probabilistic encoder) $q_{\phi}(\z|\y)$, parameterised by $\phi$, that seeks to approximate the true posterior. The most well-known AEVB model is the variational autoencoder (VAE), where the inference model as well as the probabilistic decoder $p_{\psi}(\y|\z)$ are parameterised by DNNs. Instead of optimising sample-specific variational parameters, as in standard variational inference (VI), AEVB uses amortised VI that exploits the inference model $q_{\phi}(\z|\y)$ to obtain approximate distributions for each $\z_n$ as a function of the corresponding sample $\y_n$. Then, the approximate inference problem is fitted by maximising the evidence lower bound (ELBO) of the marginal log-likelihood w.r.t.\ $\phi$:
\begin{align*} 
    &\log p_\omega(Y) \geq \mathcal{L}(\phi,\psi,\theta;Y)\nonumber\\
    &\triangleq \mathbb{E}_{q_{\phi}}[\log p_{\psi}(Y|Z)]- D_{\mathrm{KL}}(q_{\phi}(Z|Y)||p_\theta(Z)) \rightarrow \max_\phi,
\end{align*}
where $D_{\mathrm{KL}}$ denotes the KL divergence. In practice, the approximate inference is typically conducted simultaneously alongside learning the generative model's parameters via solving the joint optimisation problem $\mathcal{L}(\phi,\psi,\theta;Y)\longrightarrow \max_{\phi,\psi,\theta}$. 

\subsection{Longitudinal variational autoencoder}
The standard VAE model assumes that the joint distribution factorises across samples. Therefore, it can neither capture the potentially non-trivial structure of the data across the samples nor exploit that structure while making predictions. 
GP-VAE \citep{fortuin2019multivariate} and GPPVAE \citep{casale2018gaussian} address this issue for multivariate time-series data by assuming a GP prior on $Z$, whose CF depends on time or is factorised into feature and view components, respectively. 
However, neither of these approaches coherently accommodates population-level structure by accounting for the potentially available additional covariates (such as patient sex or genotype) with instance-specific variability. In contrast, our method naturally incorporates both using an additive multi-output GP prior.

Building upon the work of \citet{casale2018gaussian} and \citet{fortuin2019multivariate}, we propose to enhance the ability of VAEs to learn meaningful low-dimensional representations of $Y$, with a non-i.i.d.\ model for the low-dimensional latent space. This would enable our proposed method to effectively capture the structure of the data across the observed samples w.r.t.\ the auxiliary information $X$.
Specifically, our generative model, parameterised by $\omega = \{\psi,\theta\}$, is formulated as 
\begin{align}
\label{eq:gen_model}
    p_\omega(Y|X) &= \int_Z p_\psi(Y|Z,X) p_\theta(Z|X)dZ\nonumber\\ &= \int_Z\prod_{n=1}^N p_\psi(\y_n|\z_n) p_\theta(Z|X)dZ,
\end{align}
where the probabilistic decoder for normally distributed data,
\begin{equation}
\label{eq:recon_rep}
    p_\psi(\y_n|\z_n) = \mathcal{N} \left(\y_n | \g_\psi(\z_n), \text{diag}(\sigma^2_{y1},\ldots,\sigma^2_{y D}) \right)
\end{equation}
is parameterised by a neural network $\psi$ (variance parameters $\sigma_{yd}^2$ are included in $\psi$) and $p_\theta(Z|X)$ is defined by a multi-output GP prior that regulates the joint structure of $Z$ with auxiliary variables $X$.

\subsubsection{Multi-output additive GP prior}
Consider a multi-output function $\f : \X \rightarrow \Z = \R^L$ where $L > 1$, denoted as $\f(\x) = [f_1(\x),\ldots,f_L(\x)]^T$. Following \citet{alvarez2011kernels}, we denote that $\f$ follows a multi-output Gaussian process prior as $\f(\x) \sim GP(\boldsymbol{\mu}(\x), \boldsymbol{K}(\x, \x' | \theta))$, where $\boldsymbol{\mu(\x)} \in \mathbb{R}^L$ is the mean (which we assume as $\boldsymbol{0}$) and $\boldsymbol{K}(\x, \x' | \theta)$ is a matrix-valued positive definite cross-covariance function (CCF) whose entries define the covariances between the output dimensions for any $\x, \x'$. For any finite collection of inputs $X = [\x_1^T, \ldots , \x_N^T]^T$, the corresponding function values $\f(X) = [\f(\x_1)^T, \ldots, \f(\x_N)^T]^T \in \mathbb{R}^{NL \times 1}$ have a joint multivariate Gaussian distribution $p(\f(X)) = \mathcal{N}(\boldsymbol{0}, \boldsymbol{K}_{XX}(\theta))$, where the covariance matrix $\boldsymbol{K}_{XX}(\theta)$ is a block-partitioned matrix of size $NL \times NL$ with $L \times L$ blocks, so that block $[\boldsymbol{K}_{XX}(\theta)]_{i,j} = \boldsymbol{K}(\x_i, \x_j|\theta)$. 

In this work we consider multi-output GPs  that factorise across the output dimensions, which is equivalent to diagonal CCF
\begin{equation*}
    \boldsymbol{K}(\x, \x' | \theta) =  \mathrm{diag}\left(k_1(\x, \x' | \theta_1), \ldots, k_L(\x, \x' | \theta_L)\right), 
\end{equation*}
where $k_l(\x, \x' | \theta_l) = \cov(f_l(\x), f_l(\x'))$ is the CF for the $l^\mathrm{th}$ latent dimension. From the perspective of the generative model in eq.~\eqref{eq:gen_model}, such a choice is completely nonrestrictive compared to commonly used linear-type CCFs  (e.g.\ linear model of coregionalisation), since the output of the multivariate GP is multiplied with the weights of the first layer in the neural network $\psi$. Moreover, diagonal CCFs enable us to completely characterise the multi-output GP prior for multivariate latent embedding $Z$ in terms of univariate GP priors for its univariate column-components $\bar{\z}_{l}$, and to derive  an efficient inference algorithm elucidated in Sec.~\ref{sec:efficient_KLD}.

Further, we will assume an additive structure for each CF, where each component depends on only a single covariate or a pair of covariates. Such CFs naturally provide a way to handle heterogeneous inputs, interpretability, and disentanglement for different covariates. Specifically, for the $l^\mathrm{th}$ dimension we assume that
\begin{align*}
    f_l(\x) &= f_l^{(1)}(\x^{(1)}) + \ldots + f_l^{(R)}(\x^{(R)}),\\ 
    f_l^{(r)}(\x^{(r)}) &\sim GP \left( \boldsymbol{0}, k_l^{(r)}(\x^{(r)}, \x^{(r)\prime}|\theta_l^{(r)}) \right), 
\end{align*}
where $R$ refers to the number of additive kernels/components,  each $f_l^{(r)}(\x^{(r)})$ is a separate GP with specific parameters $\theta_l^{(r)}$, and $\x^{(r)} \in \mathcal{X}^{(r)} \subseteq \mathcal{X}$. The covariance function of the additive GP is a sum of its components' covariances $f_l(\x) \sim GP\left(\boldsymbol{0}, \sum_{r=1}^R k_l^{(r)}(\x^{(r)}, \x^{(r)\prime}|\theta_l^{(r)})\right)$ \citep{williams2006gaussian}. 

Assuming the $l^\mathrm{th}$ dimension of latent embedding $\bar{\z}_{l}$ is perturbed by i.i.d.\ zero-mean Gaussian noise $\sigma_{zl}^2$,  the multi-output GP marginalised likelihood follows:
\begin{align*}
    p(Z | X, \theta) &= \prod_{l=1}^L p(\bar{\z}_l | X, \theta_l,\sigma_{zl}^2)\\ &= \prod_{l=1}^L  \mathcal{N}\left(\bar{\z}_l | \boldsymbol{0}, \sum_{r=1}^{R}K_{XX}^{(l,r)}(\theta_l^{(r)}) + \sigma_{zl}^2 I_N \right),
\end{align*}
where $\theta$ incorporates all $\theta_l^{(r)}$ and $\sigma_{zl}$ parameters, and $K_{XX}^{(l,r)}(\theta_l^{(r)})$ is a $N \times N$ covariance matrix for $k_l^{(r)}(\x^{(r)}, \x^{(r)\prime}|\theta_l^{(r)})$.

Similar to~\citet{cheng2019additive}, we use the following elementary covariance functions to construct the additive GP components within our framework: a) the effects of continuous covariates are modelled with the squared exponential CF; b) categorical covariates are modelled either alone with the categorical CF or together with the time (or other continuous) covariate using an interaction CF, which is the product of the categorical and squared exponential CFs; and c) the product of the squared exponential CF and the binary CF is used to model covariates that are defined for a subset of samples (an example of such a covariate in a biomedical context is the time to the disease onset, which is defined only for those individuals who get a disease). All CFs for different latent dimensions and additive components have separate parameters $\theta_l^{(r)}$ ($LR$ in total). Moreover, each additive component has an output variance parameter, which is a scale factor that is learnt alongside the other parameters. We constrain the scale parameters to positive values and fix the likelihood noise $\sigma_{zl}^2$ to 1, for better identifiability. Restricting to this family of CFs enabled us to devise accurate approximate computational strategies to overcome the typically cubic scaling of GPs as described in Sec.~\ref{sec:efficient_KLD}. A detailed description of the elementary CFs used is included in Suppl.\ Sec.~1.  
To handle missingness in the covariates, we set each CF $k_l^{(r)}(\x^{(r)}, \x^{(r)\prime}|\theta_l^{(r)}) \doteq 0$ for those sample pairs $\x$ and $\x'$ that contain at least one missing value for their respective covariate(s) in $\X^{(r)}$. This ensures that the contribution of the missing values to the target variable is 0.

\subsubsection{Auto-Encoding Variational Bayes for L-VAE}
We approximate the true posterior of $Z$ with the product of multivariate Gaussian distributions across samples, each of which has a diagonal covariance matrix:
\begin{align}
    \label{eq:lvaeinfnet}
    q_{\phi}(Z | Y) &= \prod_{n=1}^N \mathcal{N}\left (\z_n | \boldsymbol{\mu}_{\phi}(\y_n), \text{diag}(\boldsymbol{\sigma}^2_{\phi}(\y_n))\right )\nonumber\\ &= \prod_{n=1}^N \prod_{l=1}^L \mathcal{N}\left (z_{nl} | {\mu}_{\phi,l}(\y_n), {\sigma}^2_{\phi,l}(\y_n))\right ).
\end{align}
Here, the probabilistic encoder is represented by neural network functions parameterised by $\phi$, $\boldsymbol{\mu}_{\phi} : \mathbb{R}^D \rightarrow \mathbb{R}^L$ and  $\boldsymbol{\sigma}^2_{\phi} : \mathbb{R}^D \rightarrow \mathbb{R}^L_+$, that determine the means and variances of the approximating variational distribution. Following the AEVB approach of \citet{kingma2013auto}, we form the ELBO for the L-VAE: 
\begin{align}
\label{eq:elbo_lonvae}
    &\log p_{\omega}(Y|X) \geq \mathcal{L}(\phi,\psi,\theta;Y,X)\\
    &\triangleq \mathbb{E}_{q_{\phi}(Z|Y)}\left[\log p_{\psi}(Y|Z) \right] - D_{\mathrm{KL}}(q_{\phi}(Z|Y)||p_\theta(Z|X)).\nonumber
\end{align}
From equations~(\ref{eq:gen_model}-\ref{eq:recon_rep}) and~\eqref{eq:lvaeinfnet} which describe the factorised decoder and variational approximation respectively, the first term of the ELBO in eq.~\eqref{eq:elbo_lonvae}, called reconstruction loss, consists of additive terms across the samples 
and observations which can be written as $\mathbb{E}_{q_{\phi}(Z|Y)}\left[\log p_{\psi}(Y|Z)\right] = \sum_{n=1}^N \sum_{d=1}^D \mathbb{E}_{q_{\phi}(\z_{n}|\y_n)}\left[ \log p_{\psi}(y_{nd}|\z_n) \right]$. If $Y$ contains missing values, this summation is done only over the non-missing elements.

Given that our multi-output additive GP prior factorises across the latent dimensions $p_\theta(Z|X) = \prod_{l=1}^L p(\bar{\z}_l|X,\theta_l,\sigma_{zl}^2)$, we can exploit the additive nature of the KL divergence for pairs of independent distributions:
\begin{align}
\label{eq:kld_factrorize}
    D_{\mathrm{KL}}(q_{\phi}(Z|Y,&X)||p_\theta(Z|X))\nonumber\\ &= \sum_{l=1}^L D_{\mathrm{KL}}(q_{\phi}(\bar{\z}_l|Y,X)||p_\theta(\bar{\z}_l|X)).
\end{align}
Hence, we can completely avoid explicitly dealing with the numerics of multi-output GPs.

We have experimented with, e.g., a structured variational distribution (a tri-diagonal precision matrix per longitudinal instance) motivated by Bamler and Mandt (2017). However, we did not observe any improvement in the resulting performance. We leave the exploration of other approximating variational distributions for future work.

\subsection{Efficient KL divergence computation}
\label{sec:efficient_KLD}
Optimising the variational objective in eq.~\eqref{eq:elbo_lonvae} involves the computation of the KL divergence in eq.~\eqref{eq:kld_factrorize}, which decomposes into $L$ KL divergences, 
\begin{equation}
\label{eq:kl_divergence}
\begin{aligned}
    D_{\mathrm{KL}}^{(l)} &= D_{\mathrm{KL}}(q_{\phi}(\bar{\z}_l|Y,X)||p_\theta(\bar{\z}_l|X))\\ &= D_{\mathrm{KL}}(\mathcal{N}(\bar{\boldsymbol{\mu}}_l,W_l)||\mathcal{N}(\boldsymbol{0},\Sigma_l)),
\end{aligned}
\end{equation}
where $\bar{\boldsymbol{\mu}}_{l} = [{\mu}_{\phi,l}(\y_1), \ldots, {\mu}_{\phi,l}(\y_N)]^T$, $W_l = \text{diag}({\sigma}^2_{\phi,l}(\y_1), \ldots, {\sigma}^2_{\phi,l}(\y_N))$, and $\Sigma_l = \sum_{r=1}^R K_{XX}^{(l,r)} + \sigma^2_{zl} I_N$. Each of the KL divergences is available in closed form, but its exact computation requires $\mathcal{O}(N^3)$ flops, which makes it impractical when $N$ exceeds a few thousands. Instead, we introduce a novel strategy to approximately compute this KL divergence at a significantly reduced computational cost. Without a loss of generality, we drop the index $l$ for the remainder of this section.

A closely related problem has been studied by \citet{titsias2009variational} who proposed the well-known free-form variational lower bound for a GP marginal log-likelihood $\log \mathcal{N}(\bar{\z}|\boldsymbol{0},\Sigma)$, assuming a set of $M$ inducing locations $S=[\s_1,\ldots,\s_M]$ in $\X$ and inducing function values $\u = [f(\s_1),\ldots,f(\s_M)]^T = [u_1,\ldots,u_M]^T$. In fact, any lower bound for the prior GP marginal log-likelihood induces an upper bound of KL divergence (eq.~\eqref{eq:kl_divergence}). The free-form bound of Titsias is known to be tight when $M$ is sufficiently high and the covariance function is smooth enough.

However, longitudinal studies, by definition, always contain a categorical covariate corresponding to observed instances, which makes the covariance function non-continuous. By separating the additive component that corresponds to the interaction between instances and time (or age) from the other additive components, the covariance matrix has the following general form $\Sigma = K_{XX}^{(A)} + \hat{\Sigma}$, where $\hat{\Sigma} = \text{diag}( \hat{\Sigma}_1, \dots, \hat{\Sigma}_P )$, $\hat{\Sigma}_p = K_{X_pX_p}^{(R)} + \sigma_z
^2 I_{n_p}$ and $K_{XX}^{(A)} = \sum_{r=1}^{R-1} K_{XX}^{(r)}$ contains all the other $R-1$ components. It is now clear that a reasonable inducing point set-up for the Titsias free-form bound would mandate that $M \ge P$, thus rendering the bound either computationally inefficient once $P$ is large (due to high $M$) or insufficiently tight. Since the interaction CF is essential for accurate longitudinal modelling, we devised a novel free-form divergence upper bound for this class of GPs:
\begin{equation}
\label{eq:kld_bound_novel}
\begin{aligned}
    D_{\mathrm{KL}}
    \leq  \frac{1}{2} \biggl( &\text{tr} (\bar{\Sigma}^{-1} W) + \bar{\boldsymbol{\mu}}^T \bar{\Sigma}^{-1} \bar{\boldsymbol{\mu}}
    - N + \log |\bar{\Sigma}|\\&-\log {|W|} 
    + \sum_{p=1}^P \text{tr} \left( \hat{\Sigma}_p^{-1}\tilde K_{X_pX_p}^{(A)} \right) \biggr)
\end{aligned}
\end{equation}
where $\bar{\Sigma} = K_{XS}^{(A)}K_{SS}^{(A)^{-1}}K_{SX}^{(A)} + \hat{\Sigma}$ and $\tilde K_{X_pX_p}^{(A)} = K_{X_pX_p}^{(A)} - K_{X_p S}^{(A)}K_{SS}^{(A)^{-1}}K_{SX_p}^{(A)}$.
The computational complexity of such a bound is $\mathcal{O}(\sum_{p=1}^{P}n^3_p + NM^2)$ flops, which leads to an approximately similar computational complexity as \cite{titsias2009variational} bound when $n_p \simeq M \ll N$, but is significantly tighter:
\begin{theorem}
\label{th:theorem1}
For any set of inducing points $S$, the novel bound in eq.~\eqref{eq:kld_bound_novel} is tighter than the one induced by the free-form variational bound of \citet{titsias2009variational}.
\end{theorem}
We provide a detailed derivation of the new bound and the Theorem 1 in Suppl.\ Sec.~2.
\subsection{Mini-batch training}
Although the novel KL divergence bound presented in the previous section enables $\mathcal{O}(N)$ scaling of the ELBO computation (see eq.~\eqref{eq:kld_bound_novel}), it may still be prohibitively expensive for large datasets. To address this limitation, we devised a mini-batch training scheme that allows for multiple learning steps per epoch using unbiased stochastic estimates of the ELBO and its gradient. As elaborated above, the reconstruction loss term of eq.~\eqref{eq:elbo_lonvae} trivially enables unbiased estimates. For the KL divergence part, we derived a principled mini-batching-compatible variant of the upper bound computation (eq.~\eqref{eq:kld_bound_novel}) 
by building upon the SGD training of GPs with natural gradients presented in \citep{hensman2013gaussian}. We provide an in-depth formulation and technical details in Suppl.\ Sec.~3. 

L-VAE is trained using the Adam optimiser \citep{kingma2014adam} on a {\em training} set, with early-stopping evaluated on a {\em validation} set. We evaluate imputation on the {\em training} set and predictive performance on an independent {\em test} set. See Suppl.\ Sec.~7 for more details.
\subsection{Predictive distribution}
Given the training samples $Y$, covariate information $X$, and learnt parameters $\phi, \psi, \theta$, the predictive distribution for the high-dimensional out-of-sample data $\y_*$ given covariates $\x_*$ follows
\begin{align*}
    &p_\omega(\y_*|\x_*,Y,X) \approx \int_{\z_*,Z} p_\psi(\y_*|\z_*)p_\theta(\z_*|\x_*,Z,X)\\&\qquad\qquad\qquad\qquad\qquad\qquad\qquad\quad \cdot q_\phi(Z|Y,X)d\z_* dZ\\
    &= \int_{\z_*} \prod_{d=1}^D \mathcal{N} \left(y_{*d} | g_{\psi,d}(\z_*), \sigma^2_{y d} \right)\prod_{l=1}^{L} \mathcal{N}(z_{*l}|\mu_{*l}, \sigma^2_{*l}) d\z_*
\end{align*}
where the means of the predictive low-dimensional representation are $\mu_{*l} = K_{\x_*X}^{(l)}\Sigma_{l}^{-1} \bar{\boldsymbol{\mu}}_{l}$ and variances are $\sigma^2_{*l} = k_l(\x_*,\x_*) - K_{\x_*X}^{(l)}\Sigma_l^{-1}K_{X\x_*}^{(l)} + K_{\x_*X}^{(l)}\Sigma_l^{-1}W_l\Sigma_l^{-1}K_{X\x_*}^{(l)} + \sigma_{zl}^2$. See Suppl.\ Sec.~4 for more details.

Computation of the predictive distribution above does not assume any low-rank approximation and thus scales cubically with $N$. Unlike in the training phase, the required cubic matrix operations have to be conducted only once for any number of predictions, hence for moderate-sized data it is feasible to handle exactly. However, we can also write a scalable predictive distribution for the sparse variational approximation model that uses the inducing points. Detailed expressions are shown in Suppl.\ Sec.~5.

\section{Experiments}
\label{sec:experiments}
We demonstrate the efficacy of our method in capturing the underlying data distribution and learning meaningful latent representations, quantifying performance in missing value imputation, reconstructing from the latent space, and by long-term predictions for synthetic images and a real-world medical time series dataset. 
We compare the performance of our method with various state-of-the-art methods for the respective datasets. Moreover, we have used similar encoder and decoder network architectures when performing comparisons across methods (see Suppl.\ Sec.~8). Additional results showing the latent embeddings and more comparisons can be found in Suppl.\ Sec.~9.
In all our experiments, the \textit{id} covariate is used as an identifier of an instance (e.g.\ patient or unique handwritten digit style) that takes values in $\{1,\ldots,P\}$, whereas other covariates are specific for each dataset (e.g.\ image rotation \textit{angle} or time relative to disease onset \textit{diseaseAge}). For brevity, we denote the additive components with different covariance functions (CFs) as follows: SE CF $\f_{\mathrm{se}}(\cdot)$, categorical CF $\f_{\mathrm{ca}}(\cdot)$, binary CF $\f_{\mathrm{bi}}(\cdot)$, and their interaction CFs $\f_{\mathrm{ca} \times \mathrm{se}}(\cdot \times \cdot)$ and $\f_{\mathrm{bi} \times \mathrm{se}}(\cdot \times \cdot)$. 
\begin{figure*}[t]
\begin{center}
\includegraphics[width=\textwidth]{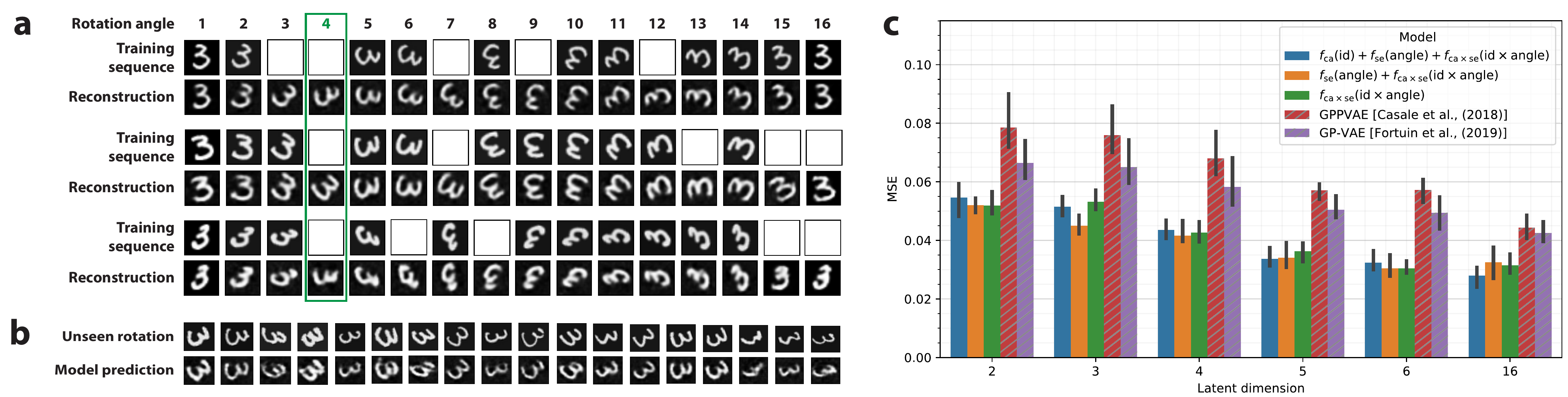}
\caption{The rotated MNIST experiments. \textbf{(a)} Reconstructions obtained from our model using $\f_{\mathrm{ca}}(\mathrm{id}) + \f_{\mathrm{se}}(\mathrm{angle}) + \f_{\mathrm{ca} \times \mathrm{se}}(\mathrm{id} \times \mathrm{angle})$, and $16$ latent dimensions. The blank boxes corresponds to the missing images. Rotation angle $4$ is completely withheld from all instances.
\textbf{(b)} Predictions for $18$ random draws of the out-of-sample prediction state (i.e. unobserved angle in panel \textbf{a} ). The first row is the real data and the bottom row is our model's prediction. \textbf{(c)} MSE on {\em test} set. 
The error bars represent the minimum and maximum values after $10$ repetitions. The {\em training}, {\em test}, and {\em validation} sets are re-sampled for each repetition.}
\label{fig:rotatedMnist_result}
\end{center}
\end{figure*}
\begin{figure*}[!t]
\begin{center}
\includegraphics[width=\textwidth]{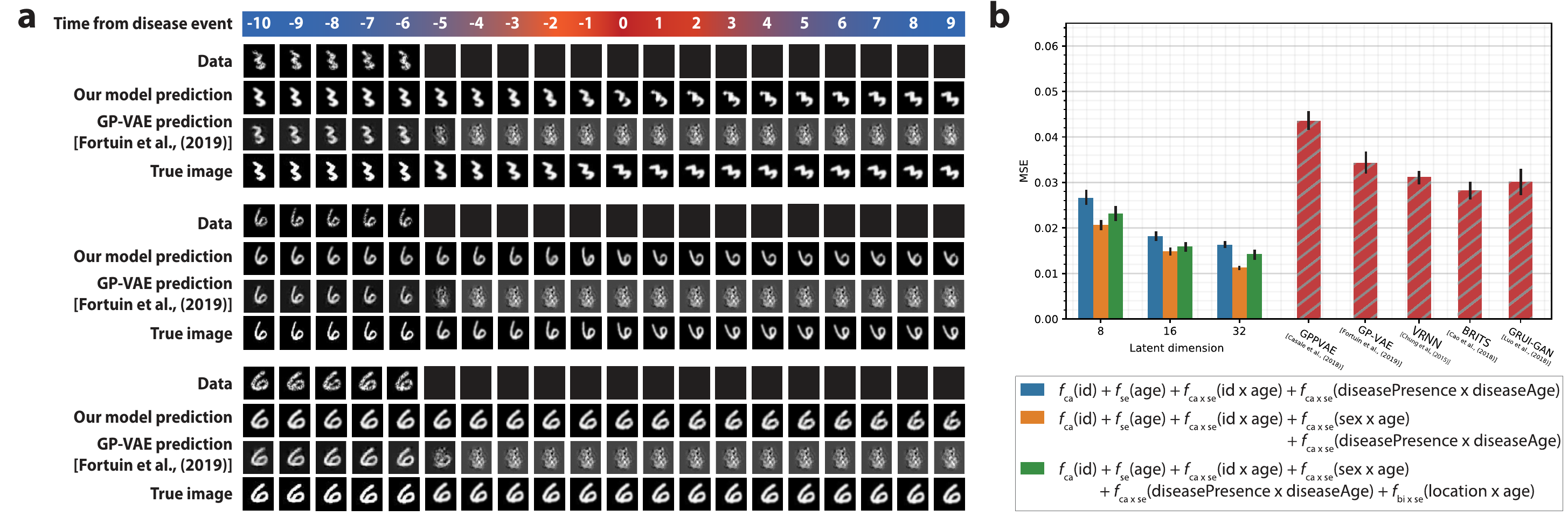}
\caption{The Health MNIST experiments. \textbf{(a)} Reconstructions and predictions obtained from our model using $\f_{\mathrm{ca}}(\mathrm{id}) + \f_{\mathrm{se}}(\mathrm{age}) + \f_{\mathrm{ca} \times \mathrm{se}}(\mathrm{id} \times \mathrm{age}) + \f_{\mathrm{ca} \times \mathrm{se}}(\mathrm{sex} \times \mathrm{age})+ \f_{\mathrm{ca} \times \mathrm{se}}(\mathrm{diseasePresence} \times \mathrm{diseaseAge})$, and $32$ latent dimensions. For the other methods, $64$ latent dimensions were used when applicable. \textbf{(b)} MSE from imputing the missing values for the observed time points. Three model variants are shown for L-VAE.}
\label{fig:healthMnist_result}
\end{center}
\end{figure*}
\subsection{Rotated MNIST digits}
We demonstrate our method on a variant of the MNIST dataset that comprises of 400 unique instances of the digit `3' as proposed in \citet{casale2018gaussian}. Each instance in this {\em training} set is rotated through $16$ evenly separated rotation angles in $[0, 2\pi)$. That is, we have two covariates: categorical \textit{id} and continuous \textit{angle}. Moreover, the {\em validation} set comprises of $40$ unique instances of the digit `3' rotated through $16$ evenly separated rotation angles like the training set.
As proposed in \citet{casale2018gaussian}, we created the {\em test} set (out-of-sample predictions) by completely removing one of the rotation angles for each instance and further  removed four randomly selected angles to simulate incomplete data.

Hence, for each sequence, 5 images are not observed (see Fig.~\ref{fig:rotatedMnist_result}(a)).
Therefore, in this experiment the {\em training} set parameters are $P=400$, $N=4400$ (where each $n_p = 11$), and $Q=2$. The {\em test} set comprises of $400$ test images of one rotation angle each. Fig.~\ref{fig:rotatedMnist_result}(a) demonstrates that our model was able to reconstruct arbitrary rotation angles, including the out-of-sample rotation (Fig.~\ref{fig:rotatedMnist_result}(b)). Fig.~\ref{fig:rotatedMnist_result}(c) compares the mean squared error (MSE) of the {\em test} set reconstructions from our method (three different GP variants) and GPPVAE and GP-VAE. The reconstruction loss decreases with increasing latent space dimension $L$, but our method consistently outperforms both the GPPVAE and GP-VAE.
\subsection{Health MNIST}
We simulate a longitudinal dataset with missing values using a modified version of the MNIST dataset. The dataset imitates many properties that may be found in actual medical data. In this experiment, we took the digits `3' as well as `6' and assumed that the different digits would represent two biological sexes. To simulate a shared age-related effect, all digit instances where shifted towards the right corner over time. We assume that half of the instances of `3' and `6' remain healthy (\textit{diseasePresence} = 0) and half get a disease (\textit{diseasePresence} = 1). For the diseased instances, we performed a sequence of $20$ rotations with the amount of rotation depending on the time to disease diagnosis (\textit{diseaseAge}).

We also introduced an irrelevant binary covariate, \textit{location} which is set randomly for each unique instance. To every data point, we applied a random rotational jitter to mimic the addition of noise. We also randomly selected $25\%$ of each image's pixels and set them as missing (we use these pixels to assess the imputation capability).  Therefore, the simulated {\em training} dataset comprises of $P=1000$ unique instances with a total of $N=20000$ samples such that each $n_p = 20$. Each sample has $Q=6$ covariates, namely \textit{age}, \textit{id}, \textit{diseasePresence}, \textit{diseaseAge}, \textit{sex}, and \textit{location}. The {\em validation} set comprised of $200$ instances which are not present in the {\em training}. Additionally, the {\em training} dataset contained $100$ additional instances for which the images of only the first $5$ time points are given (as in the `Data' row of Fig.~\ref{fig:healthMnist_result}(a)) --- we use these to assess prediction capability by computing the MSE as well as by visualising the output of the decoder. As in \citet{fortuin2019multivariate}, we try to draw an analogy to healthcare by assuming that each frame of the time series represents a collection of measurements pertaining to a patient's health state and that the temporal evolution represents the non-linear evolution of that patient's health state.

Fig.~\ref{fig:healthMnist_result}(a) indicates that our approach performs well in reconstructing the temporal trajectory (or disease trajectory as per our analogy) and is able to predict the remaining trajectory, given the corresponding covariates. The benefits of using our model can especially be seen in the time period $[-4, 9]$ as it effectively captures the non-linear transformation about the disease event. GP-VAE is also capable to effectively reconstruct in the time period $[-10,-6]$, but fails completely in future predictions because it can only utilise the \textit{age} covariate. Fig.~\ref{fig:healthMnist_result}(b) shows that our model also outperforms GP-VAE, GPPVAE, BRITS and GRUI-GAN in imputing the missing values for observed time points. 
Fig.~\ref{fig:healthMnist_result}(b) also highlights the robustness of our approach as the irrelevant covariate, \textit{location}, has a very mild detrimental effect on the overall model performance. 
Finally, Table \ref{table:healthMNIST_future} shows that L-VAE outperforms other methods in performing future predictions. 
See Suppl.\ Figs.~3 and 4 for latent space visualisations. 
\begin{table}[!t]
\caption{MSE from performing future predictions (i.e., from time [-5, 9]) on the Health MNIST dataset. The values are the means and respective standard errors.} \label{table:healthMNIST_future}
\centering
\resizebox{\linewidth}{!}{%
\begin{tabular}{lcr}
\hline 
Model    & Latent dimension & MSE             \\
\hline
GPPVAE   & 64               & $0.057 \pm 0.003$  \\
GP-VAE   & 64               & $0.059 \pm 0.002$  \\
VRNN   & 64               & $0.049 \pm 0.004$  \\
BRITS    & N/A              & $0.047 \pm 0.004$  \\
GRUI-GAN & 64               & $0.053 \pm 0.007$   \\
\hline
L-VAE    & 8                & $0.038 \pm 0.003$  \\
L-VAE    & 16               & $0.033 \pm 0.0018$ \\
L-VAE    & \textbf{32}               & $\boldsymbol{0.025 \pm 0.0015}$ \\
\hline
\end{tabular}}
\end{table}
\subsection{Healthcare data}
We evaluated our model on health-care data from the Physionet Challenge 2012 \citep{silva2012predicting}. The objective of this challenge was to predict the in-hospital mortality of the patients that were monitored in the Intensive Care Unit (ICU) over a period of 48 hours. We made use of data from 3997 individuals (`set a') for {\em training} and 1000 individuals (`set c') for {\em validation}. Additionally, we used 3993 individuals (`set b') for {\em testing}. As in \citet{cao2018brits}, we focused on modelling the measurements of $35$ different attributes (such as glucose level, blood pressure, body temperature, etc.),  approx.\ $80\%$ of which are missing in the data. We also made use of $7$ patient-specific general auxiliary covariates that were made available as a part of the challenge, i.e.\ patient identifier (\textit{id}), type of ICU unit (\textit{ICUtype}), \textit{height}, \textit{weight}, \textit{age}, \textit{sex}, in-hospital death (\textit{mortality}) as well as measurement hour (\textit{time}), some of which were also missing for some patients. We constructed an additional covariate (time to mortality or \textit{mortalityTime}) based on the provided survival time (see Suppl.\ Sec.~6 for data pre-processing). For model training, data for all patients ($P=3997)$ is available hourly ($n_p=48$), so $N=191856$.
\begin{figure}[t]
\begin{center}
\includegraphics[width=0.88\linewidth]{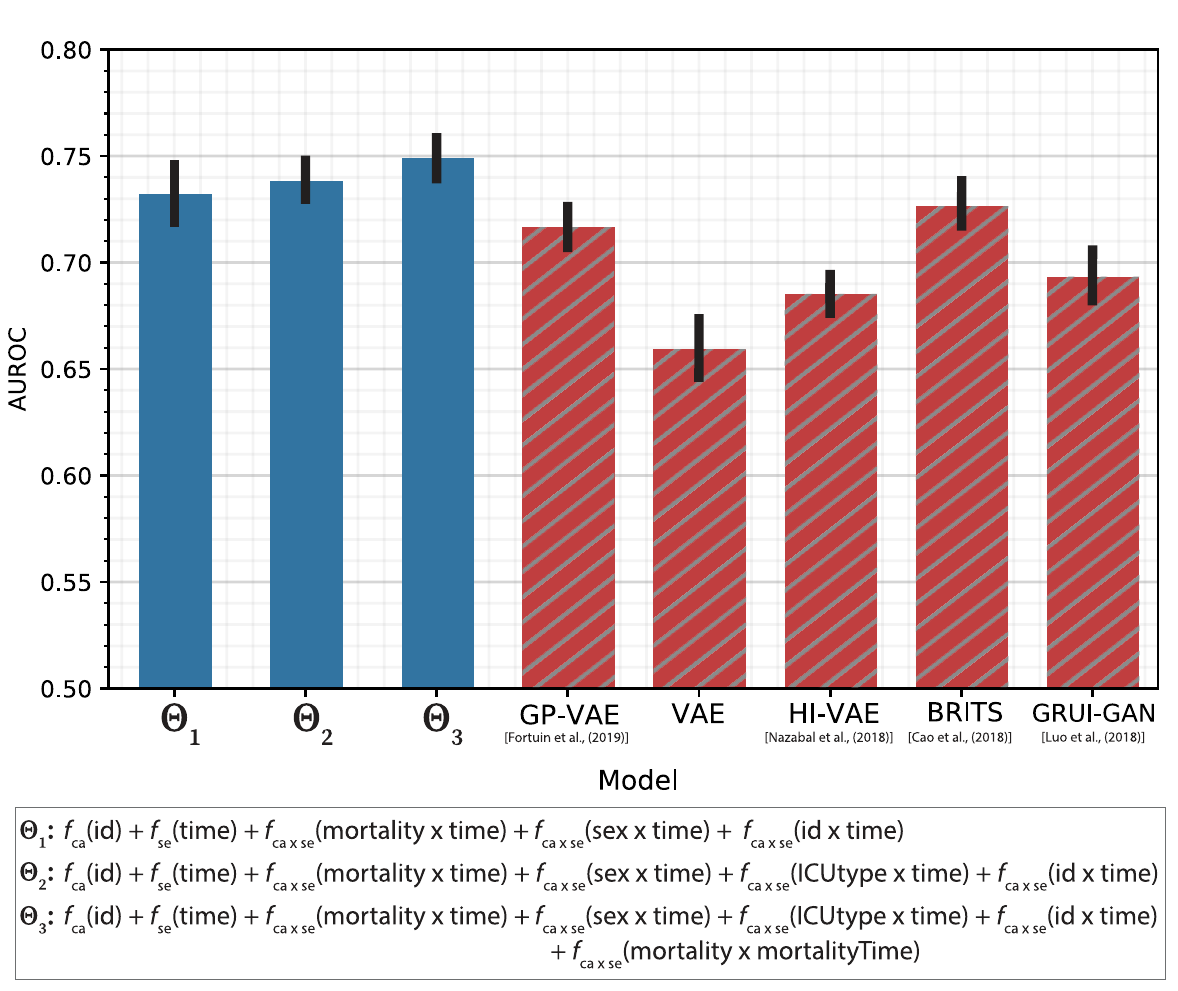}
\caption{{\em Test} set AUROC scores for the patient mortality prediction task for the Physionet Challenge 2012 dataset. The number of latent dimensions is $32$. Higher score is better. The error bars represent the minimum and maximum values after 10 repetitions.}
\label{fig:physionet_auroc}
\end{center}
\end{figure}
We trained our L-VAE model using the {\em training} samples and used it to build a Bayes classifier aimed at predicting the patient \textit{mortality} for {\em test} data. Since the {\em test} data lacks information on \textit{mortality} and \textit{mortalityTime} covariates, for each patient in the {\em test} set, characterised by a pair of 48 hour attributes time-series $Y_*$ and incomplete auxiliary information $X_*$, we approximated the marginal log-likelihoods from eq.~\eqref{eq:elbo_lonvae} of the two alternative hypotheses: $L_i = \mathcal{L}(\phi,\psi,\theta;Y_*,X_*, \text{mortality}=i)$ for $i=\{0,1\}$. Then the predicted mortality probability was computed as $P_1 = \exp(L_1)/(\exp(L_0)+\exp(L_1))$. We provide a detailed explanation of the Bayes classifier and mortality probability $P_1$ in  Suppl.\ Sec.~6.

To demonstrate the efficacy of our method, we compared the AUROC scores for predicting mortality of the {\em test} data instances obtained using our method with those obtained from GP-VAE, a standard VAE, HI-VAE \citep{nazabal2018handling}, BRITS and GRUI-GAN. These methods either do not use any auxiliary information (HI-VAE and VAE) or use only the time covariate (GP-VAE, BRITS, and GRUI-GAN). The mortality classification procedure of these methods first imputes the missing values with the generative model, and then exploits the imputed values as covariates to train a logistic regression that is finally used for mortality prediction \citep{fortuin2019multivariate}. Fig.~\ref{fig:physionet_auroc} shows that our L-VAE approach achieves higher AUROC scores. The performance with other additive GP covariance functions can be seen in Suppl.\ Fig.~5. 

\section{Discussion}
\label{sec:discussion}
In this paper, we introduced a novel deep generative model, L-VAE, that incorporates auxiliary covariate information to model the structured latent space dynamics for longitudinal datasets with missing values. We also introduced a novel computationally efficient inference strategy that exploits the structure of the additive GP covariance functions resulting in a novel lower bound. Moreover, the derived bound is theoretically guaranteed to be tighter than the free-form variational bound of \citet{titsias2009variational}. We further developed this bound to allow mini-batch SGD training for computational efficiency.  We demonstrated the efficacy of our method on synthetic as well as real-world datasets by showing that L-VAE achieves better out-of-sample prediction performance and missing value imputation than competing methods. Given the flexibility of our model and the state-of-the-art results, we expect L-VAE to become a useful tool for high-dimensional longitudinal data analysis.


\subsubsection*{Acknowledgements}
We would like to acknowledge the computational resources provided by Aalto Science-IT, Finland. We would also like to thank Charles Gadd for the helpful discussions.
This work was supported by the Academy of Finland [335436, 311584] and Business Finland [2383/31/2015].

\bibliographystyle{abbrvnat}
\bibliography{references}

\end{document}


%
\runningtitle{Longitudinal Variational Autoencoder: Supplementary Materials}

%
\runningauthor{Ramchandran, Tikhonov, Kujanp\"a\"a, Koskinen, L\"ahdesm\"aki}
\onecolumn
\aistatstitle{Longitudinal Variational Autoencoder: \\ 
Supplementary Materials}
\section{Covariance functions}
\subsection*{Squared exponential CF} Let $\x^{(r)} = x \in \X_j$ denote an univariate continuous-valued covariate. The squared exponential (SE) CF is defined as
\begin{align*}
    k_{\text{se}}(\x^{(r)}, \x^{(r)\prime}|\theta_{\text{se}}) = \sigma^2_{\text{se}}\exp\left(-\frac{(x - x')^2}{2 \l^2_{\text{se}}}\right), \qquad \theta_{\text{se}} = (\sigma^2_{\text{se}},\l_{\text{se}})
\end{align*}
where $\sigma^2_{\text{se}}$ is the magnitude parameter (also called scale) and $\l_{\text{se}} \geq 0$ is the length-scale.  The magnitude controls the marginal variance of the GP and length-scale controls its smoothness \citep{williams2006gaussian}.

\subsection*{Categorical CF} Let $\x^{(r)} = x \in \X_j$ denote a  categorical or discrete covariate. The categorical CF is defined as:
\begin{align*}
    k_{\text{ca}}(\x^{(r)}, \x^{(r)\prime}) = \begin{cases}
        1, \quad \text{if } x = x' \\
        0, \quad \text{otherwise}
    \end{cases} \qquad \theta_{\text{ca}} = \emptyset
\end{align*}

\subsection*{Binary CF} Let $\x^{(r)} = x \in \X_j$ denote an arbitrary univariate covariate. The binary CF is defined as:
\begin{align*}
    k_{\text{bi}}(\x^{(r)}, \x^{(r)\prime}) = \begin{cases}
        1, \quad \text{if } x = x' = 1 \\
        0, \quad \text{otherwise} 
    \end{cases} \qquad \theta_{\text{bi}} = \emptyset
\end{align*}

\subsection*{Interaction CF} Let $\x^{(r)} = [\x^{(a)^T},\x^{(b)^T}]^T$, where $\x^{(a)}$ and $\x^{(b)}$ are arbitrary sub-vectors of $\x$. We define the interaction CF for these subsets as the product of two CFs defined over $\x^{(a)}$ and $\x^{(b)}$ respectively:
\begin{align*}
    k_{\text{in}}(\x^{(r)}, \x^{(r)\prime}|\theta^{(r)}) = k^{(a)}(\x^{(a)}, \x^{(a)\prime}|\theta^{(a)}) k^{(b)}(\x^{(b)}, \x^{(b)\prime}|\theta^{(b)}), 
\end{align*}
where $\theta^{(r)}=\theta^{(a)} \cup \theta^{(b)}$. The interaction CF enables us to combine any combination of univariate SE, categorical, and binary CFs in a product CF. However, in practice we restrict such combinations to include no more than a single SE CF. As~\citet{cheng2019additive} stated, such a condition affords the SE GP flexibility with random intercept/slope constructions similar to the linear mixed effect modelling framework, without sacrificing interpretability. For example, in longitudinal studies an instance-specific auto-correlated temporal deviation from the population-level temporal mean can be captured by an interaction term between a categorical CF over the instance identifiers and a SE CF over the temporal covariate. Furthermore, \citet{cheng2019additive} adapted an approach to handle missing covariates based on an interaction CF containing a missing-ness mask.

\section{Efficient KL divergence computation}

As the main text states, optimising the variational objective of L-VAE involves the computation of $L$ KL divergences $D_{\mathrm{KL}}=D_{\mathrm{KL}}(\mathcal{N}(\bar{\boldsymbol{\mu}}_l,W_l)||\mathcal{N}(\boldsymbol{0},\Sigma_l))$,
where $\bar{\boldsymbol{\mu}}_l = [{\mu}_{\phi,l}(\y_1), \ldots, {\mu}_{\phi,l}(\y_N)]^T$, $W_l = \text{diag}({\sigma}^2_{\phi,l}(\y_1), \ldots, {\sigma}^2_{\phi,l}(\y_N))$, and $\Sigma_l = \sum_{r=1}^R K_{XX}^{(l,r)} + \sigma^2_{zl} I_N$. Henceforth, we drop the index $l$ for notational simplicity. Each of the KL divergences is available in closed form using the well-known expression for the KL divergence between two multivariate normal distributions:
$$
D_{\mathrm{KL}} = \frac{1}{2} \left( \text{tr} ({\Sigma}^{-1} W) + \bar{\boldsymbol{\mu}}^T {\Sigma}^{-1} \bar{\boldsymbol{\mu}}  - N + \log |{\Sigma}| - \log {|W|}  \right),
$$
but its exact computation requires $\mathcal{O}(N^3)$ flops, which makes it impractical when $N$ exceeds a few thousands. In this section, we provide a derivation of a novel strategy to approximately compute this KL divergence at a reduced computational cost.

\paragraph{KL divergence and evidence lower bound} We start by exploiting the diagonal structure of $W$ and establish the connection between the upper bound for $D_{\mathrm{KL}}$ and the evidence lower bound for the marginal log likelihood (MLL) of a Gaussian process:
\begin{align}
\label{eq:kld_bound_elbo}
    D_{\mathrm{KL}}(\mathcal{N}(\bar{\boldsymbol{\mu}},W)||\mathcal{N}(\boldsymbol{0},\Sigma)) &\triangleq \int_{\z} \mathcal{N}(\z|\bar{\boldsymbol{\mu}},W) \log \left(\frac{\mathcal{N}(\z|\bar{\boldsymbol{\mu}},W)}{\mathcal{N}(\z|\boldsymbol{0},\Sigma)} \right)  d\z \nonumber \\
    &= -\int_{\z}  \log \left(\mathcal{N}(\z|\boldsymbol{0},\Sigma) \right) \mathcal{N}(\z|\bar{\boldsymbol{\mu}},W) d\z - \frac{1}{2}\log|(2\pi e)W| \nonumber\\
    &\leq -\int_{\z} \mathcal{L}(\z;\Sigma) \mathcal{N}(\z|\bar{\boldsymbol{\mu}},W) d\z - \frac{1}{2}\log|(2\pi e)W| 
\\
& \text{for any function} \; \mathcal{L}(\z;\Sigma): \mathcal{L}(\z;\Sigma) \leq \log \left(\mathcal{N}(\z|\boldsymbol{0},\Sigma) \right) \; \forall \z. \nonumber
\end{align}
Hence, for any lower bound of GP MLL $\mathcal{L}(\z;\Sigma)$, the corresponding expression would provide an upper bound for KL divergence between the considered multivariate normal distributions. However, the lower bound of GP MLL is a much more common and well-studied problem. Please note that in the expression \eqref{eq:kld_bound_elbo} and further throughout this section, we specifically use the \textit{lower bound} and \textit{ELBO} 
terms for the GP MLL lower bound $\mathcal{L}(\z;\Sigma)$, and not for the lower bound of the deep generative model as in the main text.

\paragraph{Variational learning of inducing variables in sparse Gaussian processes} One of the most fundamental works on ELBOs for GP MLL was done by \citet{titsias2009variational}, and is based on the paradigm of low-rank inducing point approximations of GPs. We briefly recap their key results here, and later build upon their derivation by introducing modifications that would suit the specific structure of matrices $\Sigma$ in our problem setting.

We denote the set of inducing points locations in $\X$ as $S=[\s_1^T, \ldots, \s_M^T]^T$, and the value of the Gaussian process at the inducing locations as $\u=[u_1,\ldots,u_M]^T$. We recall that $\Sigma =  K_{XX} + \sigma^2_{z} I_N$ and therefore, $\z$ can be represented as the sum of noise-free GP, $\f \sim \mathcal{N}(\boldsymbol{0},K_{XX})$ and i.i.d.\ Gaussian noise. The following identities always stand:
\begin{align}
\label{eq:identities_titsias}
p(\z |  \f) &= \mathcal{N}(\z|\f,\sigma_{\z}^2 I_N) \nonumber\\
p(\f |  \u) &= \mathcal{N}(\f|K_{XS}K_{SS}^{-1}\u, \tilde{K}), \nonumber\\ 
\tilde{K} &= K_{XX}-K_{XS}K_{SS}^{-1}K_{SX} \nonumber\\
p(\u) &= \mathcal{N}(\u|\boldsymbol{0}, K_{SS}) \nonumber\\
p(\z) &= \int_{\u}\int_{\f} p(\z |  \f)  p(\f |  \u) p(\u) d\f d\u.
\end{align}
Applying the Jensen inequality on the conditional log-probability $p(\z |  \u)$ leads to:
\begin{align}
\label{eq:elbo_titsias_unmarginalized_u}
\log p(\z |  \u) &= \log \int_{\f} p(\z |  \f) p(\f |  \u) d\f \\
&\ge   \int_{\f} \log \left( p(\z |  \f) \right) p(\f |  \u) d\f = \sum_{i=1}^N \left[ \log \mathcal{N}(z_i| \mu_i ,\sigma_{\z}^2) - \frac{\tilde{K}_{ii}}{2\sigma_{\z}^2} \right],
\end{align}
where $\boldsymbol{\mu} = [\mu_1, \ldots,\mu_N]^T = K_{XS}K_{SS}^{-1}\u$ and $\tilde{K}_{ii}$ denotes the $i^\mathrm{th}$ diagonal element of $\tilde{K}$. The inequation reduces to identity \textit{iff} $\u$ is a sufficient statistic of $\f$, so that all elements of $\tilde{K}$ are zero. However, the inequation remains tight and the approximation is justified as long as the $\tilde{K}_{ii}$ elements remain small, which is achieved by setting $M$ to be sufficiently high and optimising the inducing point locations $S$. After integrating out $\u$, this approximation leads to the collapsed representation of the variational evidence lower bound,
\begin{equation}
\label{eq:elbo_titsias}
\mathcal{L}_1(\z;\Sigma) \triangleq  \log \mathcal{N}(\z|\boldsymbol{0},K_{XS}K_{SS}^{-1}K_{SX}+\sigma_{\z}^2 I_N) - \frac{1}{2\sigma^2_{\z}}\text{tr}(\tilde{K}).
\end{equation}

\paragraph{Divergence upper bound for longitudinal Gaussian process}
The free-form bound is known to be tight when $M$ is sufficiently high and the covariance function is sufficiently smooth. However, longitudinal studies, by definition, always contain a categorical covariate corresponding to instances, which makes the covariance function non-continuous. By separating the additive component that corresponds to the interaction between instances and time (or age) from the other additive components, the covariance matrix has the following general form $\Sigma = K_{XX}^{(A)} + \hat{\Sigma}$, where $\hat{\Sigma} = \text{diag}( \hat{\Sigma}_1, \dots, \hat{\Sigma}_P )$, $\hat{\Sigma}_p = K_{X_pX_p}^{(R)} + \sigma_z^2 I_{n_p}$, and $K_{XX}^{(A)} = \sum_{r=1}^{R-1} K_{XX}^{(r)}$ contains all the other $R-1$ components.
In order to keep the $\tilde{K}_{ii}$ tight, the bound from eq.~\eqref{eq:elbo_titsias} would mandate that $M \ge P$, thus rendering the bound of \citet{titsias2009variational} either computationally inefficient once $P$ is large (due to high $M$) or insufficiently tight. Since the interaction covariance function is essential for accurate longitudinal modelling, we devised a novel free-form divergence upper bound for such a class of GPs. Similar to eq.~\eqref{eq:identities_titsias}, we exploit the opportunity to represent $\z$ as a sum of noise-free GP $\f_A \sim \mathcal{N}(\boldsymbol{0},K_{XX}^{(A)})$ and structured noise $\hat{\f} \sim \mathcal{N}(\boldsymbol{0},\hat{\Sigma})$, and we assign the locations as well as values of inducing points for $\f_A$. We denote $\tilde{K}^{(A)}=K_{XX}^{(A)}-K_{XS}^{(A)}{K_{SS}^{(A)}}^{-1}K_{SX}^{(A)}$ and $\tilde{K}^{(A)}_{X_pX_p} =K_{X_pX_p}^{(A)}-K_{X_pS}^{(A)}{K_{SS}^{(A)}}^{-1}K_{SX_p}^{(A)}$, then:
\begin{align*}
\log p(\z |  \u) &= \log \int_{\f_A} p(\z | \f_A) p(\f_A | \u) d\f_A \\
&=  \log \int_{\f_A} \mathcal{N}(\z |  \f_A, \hat{\Sigma}) \mathcal{N}(\f_A|K_{XS}^{(A)}{K_{SS}^{(A)}}^{-1}\u, \tilde{K}^{(A)})  d\f_A  \\
&\geq \int_{\f_A} \log \left(\mathcal{N}(\z|  \f_A, \hat{\Sigma}) \right) \mathcal{N}(\f_A|K_{XS}^{(A)}{K_{SS}^{(A)}}^{-1}\u, \tilde{K}^{(A)})  d\f_A \\
&= \sum_{p=1}^P \left[ \log \mathcal{N} \left( \z_p |  \boldsymbol{\mu}_{p}, \hat{\Sigma}_p \right) - \frac{1}{2} \text{tr} \left( \hat{\Sigma}_p^{-1} \tilde{K}^{(A)}_{X_pX_p} \right) \right],
\end{align*}
where $\z_p$ is the sub-vector of $\z$ that corresponds to the $p^\mathrm{th}$ individual, $\boldsymbol{\mu}_{p}$ is the subvector of $K_{XS}^{(A)}{K_{SS}^{(A)}}^{-1}\u$, and $\text{tr}(\cdot)$ denotes the matrix trace operator. After integrating out $\u$, we obtain a novel variational evidence lower bound,
\begin{equation}
\label{eq:elbo_novel}
\mathcal{L}_2(\z;\Sigma) \triangleq  \log \mathcal{N}\left( \z|\boldsymbol{0},K_{XS}^{(A)}{K_{SS}^{(A)}}^{-1}K_{SX}^{(A)}+\hat{\Sigma}  \right) - \frac{1}{2} \sum_{p=1}^P \text{tr} \left( \hat{\Sigma}_p^{-1} \tilde{K}^{(A)}_{X_pX_p} \right). 
\end{equation}

We substitute this lower bound in eq.~\eqref{eq:kld_bound_elbo}, which links the ELBO and upper bound for KLD:
\begin{align}
\label{eq:kld_bound_novel}
D_{\mathrm{KL}} &= D_{\mathrm{KL}}(\mathcal{N}(\bar{\boldsymbol{\mu}},W||\mathcal{N}(\boldsymbol{0},\Sigma) )  \leq -\int_{\z} \mathcal{L}_2(\z;\Sigma) \mathcal{N}(\z|\bar{\boldsymbol{\mu}},W) d\z - \frac{1}{2}\log|(2\pi e)W|  \nonumber\\
&= -\int_{\z} \log \mathcal{N}\left( \z|K_{XS}^{(A)}{K_{SS}^{(A)}}^{-1}K_{SX}^{(A)}+\hat{\Sigma}  \right) \mathcal{N}(\z|\bar{\boldsymbol{\mu}},W) d\z - \frac{1}{2}\log|(2\pi e)W| + \frac{1}{2} \sum_{p=1}^P \text{tr} \left( \hat{\Sigma}_p^{-1} \tilde{K}^{(A)}_{X_pX_p} \right) \nonumber \\
&= D_{\mathrm{KL}} \left( \mathcal{N}(\bar{\boldsymbol{\mu}},W||\mathcal{N}(\boldsymbol{0},K_{XS}^{(A)}{K_{SS}^{(A)}}^{-1}K_{SX}^{(A)}+\hat{\Sigma} \right) + \frac{1}{2} \sum_{p=1}^P \text{tr} \left( \hat{\Sigma}_p^{-1} \tilde{K}^{(A)}_{X_pX_p} \right) \nonumber \\
&= \frac{1}{2} \left( \text{tr} (\bar{\Sigma}^{-1} W) + \bar{\boldsymbol{\mu}}^T \bar{\Sigma}^{-1} \bar{\boldsymbol{\mu}} - N + \log |\bar{\Sigma}| - \log {|W|} + \sum_{p=1}^P \text{tr} \left( \hat{\Sigma}_p^{-1}\tilde K_{X_pX_p}^{(A)} \right) \right) \triangleq \mathcal{D}_2,
\end{align}
where $\bar{\Sigma} = K_{XS}^{(A)}K_{SS}^{(A)^{-1}}K_{SX}^{(A)} + \hat{\Sigma}$.

\begin{theorem}
For any set of inducing points $S$ and $\bar{S}$ where $\bar{S}\subseteq S$, such that $\mathrm{rank}(K_{SS})=\mathrm{rank}(K^{(A)}_{\bar{S}\bar{S}})$, non-strict inequality $\mathcal{L}_1(\z;\Sigma,S) \le \mathcal{L}_2(\z;\Sigma,\bar{S})$ holds.
\end{theorem}
\begin{proof}
Based on the univariate additive GP $f(\x)=\sum_{r=1}^R f^{(r)}(\x)$, we construct another GP $g(v_1,v_2,\x) = \mathbbm{1}(v_1)\sum_{r=1}^{R-1} f^{(r)}(\x) + \mathbbm{1}(v_2)f^{(R)}(\x)$, where $\mathbbm{1}(v)$ is the indicator function that equals $1$ when $v=1$ and $0$ otherwise. Accordingly, we define augmented sets $\hat{X} = [\boldsymbol{1}_N,\boldsymbol{1}_N,X]$ and $\hat{S} = [\boldsymbol{1}_M,\boldsymbol{1}_M,S]$, where $\boldsymbol{1}_N$ is a column vector of ones that has length $N$, which effectively adds the two additional covariates, $v_1$ and $v_2$, into $X$ and $S$. Then, the marginal covariance of $g(v_1,v_2,\x)$ for $(\hat{X},\hat{X})$ is $K_{XX}$, for $(\hat{X},\hat{S})$ is $K_{XS}$, and for $(\hat{S},\hat{S})$ is $K_{SS}$. Thus, $\mathcal{L}_1(\z;\Sigma,S) = \mathcal{L}_1(\z;\Sigma,\hat{S})$.

Since the collapsed lower bound, eq.~\eqref{eq:elbo_titsias}, is obtained as a closed-form solution to the optimisation problem of the variational parameters of the distribution over the inducing points, expanding the set of the inducing points can only expand the variational family and, therefore, never decreases the $\mathcal{L}_1(\z;\Sigma,\hat{S})$ bound \citep{titsias2009variational}. We consider an expanded set of inducing points $\tilde{S} = [\hat{S}^T,\hat{S}^T_A,\hat{S}^T_B,\hat{X}^T_B]^T$, where $\hat{S}_A = [\boldsymbol{1}_M,\boldsymbol{0}_M,S]$ ($\boldsymbol{0}_M$ is a column vector of zeros that has length $M$), $\hat{S}_B = [\boldsymbol{0}_M,\boldsymbol{1}_M,S]$, and $\hat{X}_B = [\boldsymbol{0}_N,\boldsymbol{1}_N,X]$. For any $\s \in \hat{S}$, $g(1,1,\s)=g(1,0,\s)+g(0,1,\s)$, so that the values of $g(v_1,v_2,\x)$ over $\hat{S}$ are linearly dependent on the values over $\hat{S}_A$ and $\hat{S}_B$. Therefore, the variational family remains the same for the reduced set of inducing points $\check{S} = [\hat{S}^T_A,\hat{S}^T_B,\hat{X}^T_B]^T$, and  $\mathcal{L}_1(\z;\Sigma,\hat{S}) \le \mathcal{L}_1(\z;\Sigma,\tilde{S}) = \mathcal{L}_1(\z;\Sigma,\check{S})$.

We will next write the $\mathcal{L}_1$ bound from equation~(\ref{eq:elbo_titsias}) for $\check{S}$:
$$
\mathcal{L}_1(\z;\Sigma,\check{S}) = \log \mathcal{N}(\z|\boldsymbol{0},\check{K}_{X\check{S}}\check{K}_{\check{S}\check{S}}^{-1}\check{K}_{\check{S}X}+\sigma_{\z}^2 I_N) - \frac{1}{2\sigma^2_{\z}}\text{tr}(K_{XX}-\check{K}_{X\check{S}}\check{K}_{\check{S}\check{S}}^{-1}\check{K}_{\check{S}X}),
$$
where (from the definition of $\check{S}$)
$$
\check{K}_{X\check{S}} = \left[K^{(A)}_{XS},K^{(R)}_{XS},K^{(R)}_{XX}\right] \quad \text{and} \quad \check{K}_{\check{S}\check{S}} = \begin{bmatrix}
K^{(A)}_{SS} & 0 & 0\\
0 & K^{(R)}_{SS} & K^{(R)}_{SX} \\
0 & K^{(R)}_{XS} & K^{(R)}_{XX}
\end{bmatrix}.
$$
We recall the 2-by-2 block-matrix inverse formula for the second and third blocks of $\check{K}_{\check{S}\check{S}}$:
\begin{equation*}
\begin{bmatrix}
K^{(R)}_{SS} & K^{(R)}_{SX} \\
K^{(R)}_{XS} & K^{(R)}_{XX}
\end{bmatrix}^{-1} = 
\begin{bmatrix}
Q^{-1} & -Q^{-1}K^{(R)}_{SX}{K^{(R)}_{XX}}^{-1} \\
-{K^{(R)}_{XX}}^{-1}K^{(R)}_{XS}Q^{-1} & {K^{(R)}_{XX}}^{-1} + {K^{(R)}_{XX}}^{-1}K^{(R)}_{XS}Q^{-1}{K^{(R)}_{XX}}^{-1}
\end{bmatrix}
\end{equation*}
where $Q = K^{(R)}_{SS} - K^{(R)}_{SX}{K^{(R)}_{XX}}^{-1}K^{(R)}_{XS}$. Then, we substitute this expression into the matrix-product term $\check{K}_{X\check{S}}\check{K}_{\check{S}\check{S}}^{-1}\check{K}_{\check{S}X}$, alongside utilising the zero-blocks in first row/column of $\check{K}_{\check{S}\check{S}}$. Such a manipulation yields a very simple expression, as all the terms involving $Q$ are cancelled out:
\begin{align*}
\check{K}_{X\check{S}}\check{K}_{\check{S}\check{S}}^{-1}\check{K}_{\check{S}X} &= K^{(A)}_{XS}{K^{(A)}_{SS}}^{-1}K^{(A)}_{SX} + \left[K^{(R)}_{XS},K^{(R)}_{XX}\right] \begin{bmatrix}
K^{(R)}_{SS} & K^{(R)}_{SX} \\
K^{(R)}_{XS} & K^{(R)}_{XX}
\end{bmatrix}^{-1} \left[K^{(R)}_{XS},K^{(R)}_{XX}\right]^T\\ &= K^{(A)}_{XS}{K^{(A)}_{SS}}^{-1}K^{(A)}_{SX} + K^{(R)}_{XX}.
\end{align*}
The $\mathcal{L}_1$ bound for the inducing point set $\check{S}$ can then be written as:
\begin{align}
\mathcal{L}_1(\z;\Sigma,\check{S}) &= \log \mathcal{N}\left(\z|\boldsymbol{0},K^{(A)}_{XS}{K^{(A)}_{SS}}^{-1}K^{(A)}_{SX} + K^{(R)}_{XX}+\sigma_{\z}^2 I_N\right) - \frac{1}{2\sigma^2_{\z}}\text{tr}(K_{XX}-K^{(A)}_{XS}{K^{(A)}_{SS}}^{-1}K^{(A)}_{SX} - K^{(R)}_{XX})\nonumber \\
&= \log \mathcal{N}\left( \z|\boldsymbol{0},K_{XS}^{(A)}{K_{SS}^{(A)}}^{-1}K_{SX}^{(A)}+\hat{\Sigma}  \right) - \frac{1}{2\sigma^2_{\z}}\text{tr}(K^{(A)}_{XX}-K^{(A)}_{XS}{K^{(A)}_{SS}}^{-1}K^{(A)}_{SX}).\label{eq:L1checkS}
\end{align}
Focusing on the last trace term we have,
\begin{align}
\text{tr} \left( \hat{\Sigma}^{-1} \tilde{K}^{(A)}_{XX} \right) &= \text{tr} \left( (K^{(R)}_{XX}+\sigma_{\z}^2 I_N)^{-1} \tilde{K}^{(A)}_{XX} \right) = \text{tr} \left( (\sigma_{\z}^{-2} I_N - \sigma_{\z}^{-4}({K^{(R)}_{XX}}^{-1}+\sigma_{\z}^{-2} I_N)^{-1}) \tilde{K}^{(A)}_{XX} \right) \nonumber \\
&= \sigma_{\z}^{-2}\text{tr} \left( \tilde{K}^{(A)}_{XX} \right) - \sigma_{\z}^{-4}\text{tr} \left( ({K^{(R)}_{XX}}^{-1}+\sigma_{\z}^{-2} I_N)^{-1} \tilde{K}^{(A)}_{XX} \right)\nonumber \\ 
&= \sigma_{\z}^{-2}\text{tr} \left( \tilde{K}^{(A)}_{XX} \right) - \sigma_{\z}^{-4}\text{tr} \left( \hat{L}\hat{L}^T \tilde{K}^{(A)}_{XX} \right) = \sigma_{\z}^{-2}\text{tr} \left( \tilde{K}^{(A)}_{XX} \right) - \sigma_{\z}^{-4}\text{tr} \left( \hat{L}^T \tilde{K}^{(A)}_{XX} \hat{L} \right)\nonumber \\ 
&\le \sigma_{\z}^{-2}\text{tr} \left( \tilde{K}^{(A)}_{XX} \right) \label{eq:trace-ineq}\\
&= \frac{1}{\sigma^2_{\z}}\text{tr}(K^{(A)}_{XX}-K^{(A)}_{XS}{K^{(A)}_{SS}}^{-1}K^{(A)}_{SX})\nonumber,
\end{align}
where $\hat{L}$ is the Cholesky decomposition of $({K^{(R)}_{XX}}^{-1}+\sigma_{\z}^{-2} I_N)^{-1}$. As matrix  ${K^{(R)}_{XX}}^{-1}+\sigma_{\z}^{-2} I_N$ is always positive definite, its inverse and corresponding Cholesky decomposition always exist. Additionally, we used the matrix trace cyclic rotation property and the fact that the trace of a positive semidefinite matrix is $\hat{L}^T \tilde{K}^{(A)}_{XX} \hat{L}$, which is always non-negative. 
Combining equations~(\ref{eq:L1checkS}) and~(\ref{eq:trace-ineq}) we have 
\begin{align*}
    \mathcal{L}_1(\z;\Sigma,\check{S}) &= \log \mathcal{N}\left( \z|\boldsymbol{0},K_{XS}^{(A)}{K_{SS}^{(A)}}^{-1}K_{SX}^{(A)}+\hat{\Sigma}  \right) - \frac{1}{2\sigma^2_{\z}}\text{tr}(K^{(A)}_{XX}-K^{(A)}_{XS}{K^{(A)}_{SS}}^{-1}K^{(A)}_{SX})\\ 
    &\le \log \mathcal{N}\left( \z|\boldsymbol{0},K_{XS}^{(A)}{K_{SS}^{(A)}}^{-1}K_{SX}^{(A)}+\hat{\Sigma}  \right) - \frac{1}{2}\text{tr} \left( \hat{\Sigma}^{-1} \tilde{K}^{(A)}_{XX} \right)\\
    &= \mathcal{L}_2(\z;\Sigma,S).
\end{align*}
Consequently, the non-strict inequalities chain gives the final result:
$$
\mathcal{L}_1(\z;\Sigma,S) = \mathcal{L}_1(\z;\Sigma,\hat{S}) \le \mathcal{L}_1(\z;\Sigma,\tilde{S}) = \mathcal{L}_1(\z;\Sigma,\check{S}) \le \mathcal{L}_2(\z;\Sigma,S)
$$
\end{proof}

\paragraph{Computational complexity}
The computational complexity of the bound described in eq.~\eqref{eq:kld_bound_novel}, is $\mathcal{O}(\sum_{p=1}^{P}n^3_p + NM^2)$ flops, which leads to an approximately similar computational complexity as the \cite{titsias2009variational} bound, when $n_p \simeq M \ll N$, but is significantly tighter. Here, we elucidate in detail how to perform the computations in eq.~\eqref{eq:kld_bound_novel} to achieve such complexity. The first term of eq.~\eqref{eq:kld_bound_novel} can be written as:
\begin{align*}
\text{tr} (\bar{\Sigma}^{-1} W) &= \text{tr} (\hat{\Sigma}^{-1} W) - \text{tr} \left( \hat{\Sigma}^{-1}K_{XS}^{(A)} \left[ K_{SS}^{(A)} +  K_{SX}^{(A)}\hat{\Sigma}^{-1}K_{XS}^{(A)} \right]^{-1} K_{SX}^{(A)}\hat{\Sigma}^{-1} W \right) \\
 &= \sum_{p=1}^P \left(\text{diag}(\hat{\Sigma}^{-1}) \cdot \text{diag}(W_p) \right) - \text{tr} \left(  V^{-1} \left( K_{SX}^{(A)}\hat{\Sigma}^{-1} W  \hat{\Sigma}^{-1} K_{XS}^{(A)} \right) \right),
\end{align*}
where $V=K_{SS}^{(A)} +  K_{SX}^{(A)}\hat{\Sigma}^{-1}K_{XS}^{(A)}$. We have used the Woodbury matrix identity to obtain the first equality and the cyclic rotation property of the matrix trace to obtain the second equality \citep{press2007numerical}. Since $\hat{\Sigma}$ is a block-diagonal matrix, obtaining $\hat{\Sigma}^{-1}$ takes  $\mathcal{O}(\sum_{p=1}^{P}n^3_p)$ flops and obtaining products $K_{SX}^{(A)}\hat{\Sigma}^{-1}K_{XS}^{(A)}$ or $K_{SX}^{(A)}( \hat{\Sigma}^{-1} W  \hat{\Sigma}^{-1}) K_{XS}^{(A)}$ takes $\mathcal{O}(\sum_{p=1}^{P}n_pM^2)=\mathcal{O}(NM^2)$ flops. Moreover, inverting $V$ takes $\mathcal{O}(M^3)$ flops.

We can use the Woodbury matrix identity to write the second term as,
$$
\bar{\boldsymbol{\mu}}^T \bar{\Sigma}^{-1} \bar{\boldsymbol{\mu}} = \bar{\boldsymbol{\mu}}^T \hat{\Sigma}^{-1} \bar{\boldsymbol{\mu}} - \bar{\boldsymbol{\mu}}^T  \hat{\Sigma}^{-1}K_{XS}^{(A)} V^{-1} K_{SX}^{(A)}\hat{\Sigma}^{-1} \bar{\boldsymbol{\mu}}.
$$
Obtaining $\hat{\Sigma}^{-1} \bar{\boldsymbol{\mu}}$ takes $\mathcal{O}(\sum_{p=1}^{P}n^2_p)$ flops and obtaining $K_{SX}^{(A)}\left( \hat{\Sigma}^{-1} \bar{\boldsymbol{\mu}} \right)$ takes $\mathcal{O}(NM^2)$ flops.

For the forth term we use the generalised determinant lemma so that,
$$
|\bar{\Sigma}|  = |\hat{\Sigma}| |K_{SS}^{(A)}|^{-1} |V|.
$$
These determinant computations take $\mathcal{O}(\sum_{p=1}^{P}n^3_p)$, $\mathcal{O}(M^3)$, and $\mathcal{O}(M^3)$, respectively.

The fifth term is trivially $\mathcal{O}(N)$. We can again use the cyclic rotation property of the matrix trace to write the last term as:
\begin{align*}
\text{tr} \left( \hat{\Sigma}_p^{-1}\tilde K_{X_pX_p}^{(A)} \right) &= \text{tr} \left( \hat{\Sigma}_p^{-1}K_{X_pX_p}^{(A)} \right) - \text{tr} \left( \hat{\Sigma}_p^{-1}K_{X_pS}^{(A)}{K_{SS}^{(A)}}^{-1}K_{SX_p}^{(A)} \right) \\
&= \text{tr} \left( \hat{\Sigma}_p^{-1}K_{X_pX_p}^{(A)} \right) - \text{tr} \left( \left( K_{SX_p}^{(A)}\hat{\Sigma}_p^{-1}K_{X_pS}^{(A)} \right) {K_{SS}^{(A)}}^{-1} \right).
\end{align*}
The first trace takes $\mathcal{O}(n^2_p)$ to compute once the $\hat{\Sigma}^{-1}$ is available. It takes $\mathcal{O}(n_pM^2)$ to compute the product in parenthesis in the second trace, and $\mathcal{O}(M^2)$ for the trace itself once ${K_{SS}^{(A)}}^{-1}$ is available. Since we need to compute these traces for $p=1, \ldots, P$, the overall complexity of the last term in eq.~\eqref{eq:kld_bound_novel} is $\mathcal{O}(\sum_{p=1}^{P}n^2_p + NM^2)$.

Combining the terms we get the final time complexity: 
\begin{eqnarray}
\text{Complexity} \ &=& \mathcal{O}(\sum_{p=1}^{P}n^3_p) +  \mathcal{O}(NM^2) + \mathcal{O}(M^3) + \mathcal{O}(\sum_{p=1}^{P}n^2_p) + \mathcal{O}(NM^2)  \nonumber\\ 
&& + \mathcal{O}(\sum_{p=1}^{P}n^3_p) + \mathcal{O}(M^3) + \mathcal{O}(M^3) + \mathcal{O}(N) + \mathcal{O}(\sum_{p=1}^{P}n^2_p + NM^2) \nonumber\\ 
&=& \mathcal{O}(\sum_{p=1}^{P}n^3_p + NM^2).
\end{eqnarray}

\section{Stochastic Variational Inference for longitudinal Gaussian process}

A fundamental drawback of the methods described in the previous section is that they require using the full training dataset to compute the loss and perform a gradient step. This can be an issue for many problems with large data, such as sequences of images, electronic health records, etc. A common machine learning technique to tackle such problems is based on training the model using mini-batches. The mini-batch approach makes use of unbiased stochastic estimates of the loss and its gradients, which are computed based on a subset of the data. The subsets are chosen such that all training data points are used within an epoch. In this section, we first recall earlier work on how to adjust the bound of eq.~\eqref{eq:elbo_titsias} to become compatible with stochastic variational inference and mini-batching. Then, we modify this bound to account for the specifics of the GP covariance structure in L-VAE.

In contrast to \citet{titsias2009variational}, \citet{hensman2013gaussian} proposed to avoid analytical marginalisation of inducing values $\u$ in eq.~\eqref{eq:elbo_titsias_unmarginalized_u}. Instead, \citet{hensman2013gaussian} proposed to explicitly keep track of its distribution, which is assumed to be Gaussian $\u \sim \mathcal{N}(\m,H)$. Then, the authors derived an alternative evidence lower bound:
\begin{align}
\label{eq:elbo_hensman}
    \log p(\z) &= \int_{\u}\int_{\f} p(\z |  \f)  p(\f |  \u) p(\u) d\f d\u = \log \int_{\u} p(\z |  \u) \frac{p(\u)}{q(\u)} q(\u) d\u \ge \int_{\u} \log p(\z | \u) q(\u) d\u - D_{KL}(q(\u)||p(\u)) \nonumber
    \\
    &\geq \int_{\u} \log \mathcal{N}(\z|K_{XS}K_{SS}^{-1}\u,\sigma_z^2 I_N) \mathcal{N}(\u|\m,H) d\u - \frac{1}{2\sigma_z^2}\tr(\tilde{K}) - D_{KL}(\mathcal{N}(\m,H) || \mathcal{N}(\boldsymbol{0},K_{SS})) \nonumber
    \\
    &= \log \mathcal{N}(\z | K_{XS}K_{SS}^{-1}\m,\sigma_z^2 I_N) - \frac{1}{2\sigma_z^2}\tr(HK_{SS}^{-1}K_{SX}K_{XS}K_{SS}^{-1}) - \frac{1}{2\sigma_z^2}\tr(\tilde{K}) -  D_{KL}(\mathcal{N}(\m,H) || \mathcal{N}(\boldsymbol{0},K_{SS})) \nonumber\\ &\triangleq \mathcal{L}_3.
\end{align}
Substituting this bound in eq.~\eqref{eq:kld_bound_elbo} yields:
\begin{align}
D_{\mathrm{KL}}(\mathcal{N}(\bar{\boldsymbol{\mu}},W)||\mathcal{N}(\boldsymbol{0},\Sigma)) &\leq -\int_{\z} \mathcal{L}_3 \mathcal{N}(\z|\bar{\boldsymbol{\mu}},W) d\z - \frac{1}{2}\log|(2\pi e)W| \nonumber \\
&= -\int_{\z}\log \mathcal{N}(\z | K_{XS}K_{SS}^{-1}\m,\sigma_z^2 I_N) \mathcal{N}(\z|\bar{\boldsymbol{\mu}},W) d\z + \frac{1}{2\sigma_z^2}\tr(\tilde{K})\nonumber\\ & \quad + \frac{1}{2\sigma_z^2}\tr(HK_{SS}^{-1}K_{SX}K_{XS}K_{SS}^{-1})+ D_{KL}(\mathcal{N}(\m,H) || \mathcal{N}(\boldsymbol{0},K_{SS}))  - \frac{1}{2}\log|(2\pi e)W| \nonumber\\
&= \frac{1}{2}(K_{XS}K_{SS}^{-1}\m - \bar{\boldsymbol{\mu}})^T (\sigma_z^2 I_N)^{-1} (K_{XS}K_{SS}^{-1}\m - \bar{\boldsymbol{\mu}}) + \frac{1}{2}\tr((\sigma_z^2 I_N)^{-1}W) + \frac{N}{2} \log \sigma_z^2 \nonumber \\ &\quad +\frac{N}{2} \log 2\pi 
+ \frac{1}{2\sigma_z^2}\tr(\tilde{K}) + \frac{1}{2\sigma_z^2}\tr(HK_{SS}^{-1}K_{SX}K_{XS}K_{SS}^{-1}) \nonumber \\ &\quad + D_{KL}(\mathcal{N}(\m,H) || \mathcal{N}(\boldsymbol{0},K_{SS}))  - \frac{1}{2}\log|(2\pi e)W| \nonumber \\
&= \frac{1}{2}\biggl( \sigma_z^{-2} \sum_{i=1}^N{(K_{\x_iS}K_{SS}^{-1}\m - \bar{\mu}_i)^2} + \sigma_z^{-2} \sum_{i=1}^N{\sigma_\phi^2(\y_i)} + N \log \sigma_z^2 + \sigma_z^{-2} \sum_{i=1}^N{\tilde{K}_{ii}}\nonumber\\ &\quad + \sigma_z^{-2} \sum_{i=1}^N{\tr\left( \left(K_{SS}^{-1}HK_{SS}^{-1} \right) \left(K_{S\x_i}K_{\x_iS} \right)\right)}
- \sum_{i=1}^N{\log \sigma_\phi^2(\y_i)} - N \biggl)\nonumber\\  &\quad + D_{KL}(\mathcal{N}(\m,H) || \mathcal{N}(\boldsymbol{0},K_{SS}))\nonumber\\  &\triangleq \mathcal{D}_3 
\end{align}
Each term, except the last one, is additive over $i=1,\ldots, N$. Therefore, replacing the sum over all $i=1,\ldots, N$ with a batch-normalised partial sum over a subset of indices, $\mathcal{I} \subset \{1, \ldots, N\}$ of size $|\mathcal{I}| = \hat{N}$:
\begin{align}
    \hat{\mathcal{D}}_3 &= \frac{1}{2}\frac{N}{\hat{N}} \sum_{i \in \mathcal{I}}{\left( \sigma_z^{-2} (K_{\x_iS}K_{SS}^{-1}\m - \bar{\boldsymbol{\mu}}_i)^2 + \sigma_z^{-2} \sigma_\phi^2(\y_i) + \sigma_z^{-2} \tilde{K}_{ii} + \sigma_z^{-2} \tr\left( \left(K_{SS}^{-1}HK_{SS}^{-1} \right) \left(K_{S\x_i}K_{\x_iS} \right)\right) - \log \sigma_\phi^2(\y_i) \right)} \nonumber \\ 
    &\quad + \frac{N}{2} \log \sigma_z^2  - \frac{N}{2} + D_{KL}(\mathcal{N}(\m,H) || \mathcal{N}(\boldsymbol{0},K_{SS})), 
\end{align}
is an unbiased estimate of the KL divergence upper bound $E_{\mathcal{I} \sim \mathfrak{S}\{1, \ldots, N\}} (\hat{\mathcal{D}}_3) = \mathcal{D}_3 \geq D_{\mathrm{KL}}(\mathcal{N}(\bar{\boldsymbol{\mu}},W)||\mathcal{N}(\boldsymbol{0},\Sigma))$. Here, $\mathfrak{S}\{1, \ldots, N\}$ is a uniform distribution over elements of arbitrary fixed partitions of set $\{1, \ldots, N\}$. This property enables us to use the mini-batching technique for the approximate computation of the KL divergence term of L-VAE and its gradients.

Similar to our criticism of the ELBO $\mathcal{L}_1$, the $\mathcal{L}_3$ is not well suited to the typical properties of GP covariance structures used for longitudinal modelling. Following the same notation as in the previous section, we introduce a modification to eq.~\eqref{eq:elbo_hensman}:
\begin{align}
\label{eq:elbo_novel_hensman}
    \log p(\z) &= \int_{\u}\int_{\f} p(\z |  \f)  p(\f |  \u) p(\u) d\f d\u = \log \int_{\u} p(\z |  \u) \frac{p(\u)}{q(\u)} q(\u) d\u \geq \int_{\u} \log p(\z | \u) q(\u) d\u - D_{KL}(q(\u)||p(\u)) \nonumber
    \\
    &\geq \int_{\u} \log \mathcal{N}(\z|K_{XS}^{(A)}{K_{SS}^{(A)}}^{-1}\u,\hat{\Sigma}) \mathcal{N}(\u|\m,H) d\u - \frac{1}{2}\sum_{p=1}^P \text{tr} \left( \hat{\Sigma}_p^{-1} \tilde{K}^{(A)}_{X_pX_p} \right) - D_{KL}(\mathcal{N}(\m,H) || \mathcal{N}(\boldsymbol{0},K_{SS}^{(A)})) \nonumber
    \\ 
    &= \log \mathcal{N}(\z | K_{XS}^{(A)}{K_{SS}^{(A)}}^{-1}\m, \hat{\Sigma}) - \frac{1}{2}\tr(H{K_{SS}^{(A)}}^{-1}K_{SX}^{(A)}\hat{\Sigma}^{-1}K_{XS}^{(A)}{K_{SS}^{(A)}}^{-1}) - \frac{1}{2}\sum_{p=1}^P \text{tr} \left( \hat{\Sigma}_p^{-1} \tilde{K}^{(A)}_{X_pX_p} \right)\nonumber\\ &\quad-  D_{KL}(\mathcal{N}(\m,H) || \mathcal{N}(\boldsymbol{0},K_{SS}^{(A)})) \nonumber \\ &\triangleq \mathcal{L}_4
\end{align}
Substituting this novel ELBO in eq.~\eqref{eq:kld_bound_elbo} yields:
\begin{align}
D_{\mathrm{KL}}(\mathcal{N}(\bar{\boldsymbol{\mu}},W)||\mathcal{N}(\boldsymbol{0},\Sigma)) &\leq -\int_{\z} \mathcal{L}_4 \mathcal{N}(\z|\bar{\boldsymbol{\mu}},W) d\z - \frac{1}{2}\log|(2\pi e)W| \nonumber \\ &= -\int_{\z}\log \mathcal{N}(\z | K_{XS}^{(A)}{K_{SS}^{(A)}}^{-1}\m, \hat{\Sigma})\mathcal{N}(\z|\bar{\boldsymbol{\mu}},W) d\z + \frac{1}{2}\sum_{p=1}^P \text{tr} \left( \hat{\Sigma}_p^{-1} \tilde{K}^{(A)}_{X_pX_p} \right) \nonumber \\&\quad + \frac{1}{2}\tr(H{K_{SS}^{(A)}}^{-1}K_{SX}^{(A)}\hat{\Sigma}^{-1}K_{XS}^{(A)}{K_{SS}^{(A)}}^{-1}) + D_{KL}(\mathcal{N}(\m,H) || \mathcal{N}(\boldsymbol{0},K_{SS}^{(A)}))\nonumber \\&\quad  - \frac{1}{2}\log|(2\pi e)W| \nonumber
\\
&= \frac{1}{2}(K_{XS}^{(A)}{K_{SS}^{(A)}}^{-1}\m - \bar{\boldsymbol{\mu}})^T \hat{\Sigma}^{-1} (K_{XS}^{(A)}{K_{SS}^{(A)}}^{-1}\m - \bar{\boldsymbol{\mu}}) + \frac{1}{2}\tr(\hat{\Sigma}^{-1}W) + \frac{1}{2} \log |\hat{\Sigma}| \nonumber\\ &\quad + \frac{N}{2} \log 2\pi
+ \frac{1}{2}\sum_{p=1}^P \text{tr} \left( \hat{\Sigma}_p^{-1} \tilde{K}^{(A)}_{X_pX_p} \right) + \frac{1}{2}\tr(H{K_{SS}^{(A)}}^{-1}K_{SX}^{(A)}\hat{\Sigma}^{-1}K_{XS}^{(A)}{K_{SS}^{(A)}}^{-1}) \nonumber \\&\quad + D_{KL}(\mathcal{N}(\m,H) || \mathcal{N}(\boldsymbol{0},K_{SS}^{(A)}))  - \frac{1}{2}\log|(2\pi e)W| \nonumber
\\
&= \frac{1}{2}\biggl( \sum_{p=1}^P {(K_{X_pS}^{(A)}{K_{SS}^{(A)}}^{-1}\m - \hat{\boldsymbol{\mu}}_p)^T \hat{\Sigma}_p^{-1} (K_{X_pS}^{(A)}{K_{SS}^{(A)}}^{-1}\m - \hat{\boldsymbol{\mu}}_p)}\nonumber \\&\quad + \sum_{p=1}^P \sum_{i=1}^{n_p}{(\hat{\Sigma}_p^{-1})}_{ii} \sigma_\phi^2(\y_{\mathcal{I}_{pi}}) + \sum_{p=1}^P \log |\hat{\Sigma}_p| + \frac{1}{2}\sum_{p=1}^P \text{tr} \left( \hat{\Sigma}_p^{-1} \tilde{K}^{(A)}_{X_pX_p} \right)\nonumber \\ & \quad + \sum_{p=1}^P{\tr\left( \left({K_{SS}^{(A)}}^{-1}H{K_{SS}^{(A)}}^{-1} \right) \left(K_{SX_p}^{(A)}\hat{\Sigma}_p^{-1}K_{X_pS}^{(A)} \right)\right)} - \sum_{i=1}^N{\log \sigma_\phi^2(\y_i)} - N \biggl) \nonumber\\ &\quad + D_{KL}(\mathcal{N}(\m,H) || \mathcal{N}(\boldsymbol{0},K_{SS}^{(A)})) \nonumber \\ &\triangleq \mathcal{D}_4,
\end{align}
where $\mathcal{I}_{pi}$ is the index of the $i$\textsuperscript{th} sample for the $p$\textsuperscript{th} patient and $\hat{\boldsymbol{\mu}}_p = [\bar{\mu}_{\mathcal{I}_{p1}},\ldots,\bar{\mu}_{\mathcal{I}_{pn_p}}]^T$ is a sub-vector of $\bar{\boldsymbol{\mu}}$ that corresponds to the $p$\textsuperscript{th} patient. Each term, except for the last one, is additive over $p=1,\ldots, P$. Therefore, replacing the sum over all $p=1,\ldots, P$ with a batch-normalised partial sum over a subset of indices $\mathcal{P} \subset \{1, \ldots, P\}$ of size $|\mathcal{P}| = \hat{P}$:
\begin{align}
    \hat{\mathcal{D}}_4 &= \frac{1}{2}\frac{P}{\hat{P}} \sum_{p \in \mathcal{P}} \Biggl( (K_{X_pS}^{(A)}{K_{SS}^{(A)}}^{-1}\m - \hat{\boldsymbol{\mu}}_p)^T \hat{\Sigma}_p^{-1} (K_{X_pS}^{(A)}{K_{SS}^{(A)}}^{-1}\m - \hat{\boldsymbol{\mu}}_p) + \sum_{i=1}^{n_p}{(\hat{\Sigma}_p^{-1})}_{ii} \sigma_\phi^2(\y_{\mathcal{I}_{pi}}) + \log |\hat{\Sigma}_p| \nonumber\\ &\quad + \text{tr} \left( \hat{\Sigma}_p^{-1} \tilde{K}^{(A)}_{X_pX_p} \right) + 
    \tr\left( \left({K_{SS}^{(A)}}^{-1}H{K_{SS}^{(A)}}^{-1} \right) \left(K_{SX_p}^{(A)}\hat{\Sigma}_p^{-1}K_{X_pS}^{(A)} \right)\right) - \sum_{i=1}^{n_p}\log \sigma_\phi^2(\y_{\mathcal{I}_{pi}}) \Biggl)\nonumber\\ &\quad - \frac{N}{2} + D_{KL}(\mathcal{N}(\m,H) || \mathcal{N}(\boldsymbol{0},K_{SS}^{(A)})), 
\end{align}
is an unbiased estimate of the KL divergence upper bound $E_{\mathcal{P} \sim \mathfrak{S}\{1, \ldots, P\}} (\hat{\mathcal{D}}_4) = \mathcal{D}_4 \geq D_{\mathrm{KL}}(\mathcal{N}(\bar{\boldsymbol{\mu}},W)||\mathcal{N}(\boldsymbol{0},\Sigma))$. This property enables us to use the mini-batching technique for a more precise approximate computation of the KL divergence term of L-VAE and its gradients, by approximately splitting equal number of patients to each batch.

The learning of variational parameters $\m$ and $H$ can be done either by explicitly parameterising and updating them within the overall optimisation scheme, or by using the natural gradients approach similar to \citet{hensman2013gaussian}. Similar to the claims of the original paper, using the natural gradients can mitigate the challenge of finding a robust unconstrained parameterisation for the variational parameters. The gradients of KL divergence unbiased estimate $\hat{\mathcal{D}}_4$ w.r.t.\ the variational parameters have the following form:
\begin{align}
    \frac{\partial \hat{\mathcal{D}}_4}{\partial \m} &= -\sum_{p \in \mathcal{P}}{K_{SS}^{(A)}}^{-1} K_{SX_p}^{(A)}\hat{\Sigma}_p^{-1}\hat{\boldsymbol{\mu}}_p + \left( \sum_{p \in \mathcal{P}}{K_{SS}^{(A)}}^{-1} K_{SX_p}^{(A)}\hat{\Sigma}_p^{-1}K_{X_pS}^{(A)} {K_{SS}^{(A)}}^{-1} + {K_{SS}^{(A)}}^{-1} \right) \m \\
    \frac{\partial \hat{\mathcal{D}}_4}{\partial H} &= -\frac{1}{2}H^{-1} + \frac{1}{2}\sum_{p \in \mathcal{P}}{K_{SS}^{(A)}}^{-1} K_{SX_p}^{(A)}\hat{\Sigma}_p^{-1}K_{X_pS}^{(A)} {K_{SS}^{(A)}}^{-1} + \frac{1}{2} {K_{SS}^{(A)}}^{-1}
\end{align}

Making use of the chain rule, we can write the following:
$$
\frac{\partial \mathcal{D}_4}{\partial \boldsymbol{\eta}} = \frac{\partial \mathcal{D}_4}{\partial [m,H]} \frac{\partial [m,H]}{\partial \boldsymbol{\eta}} = \begin{bmatrix}\frac{\partial \mathcal{D}_4}{\partial m} & \frac{\partial \mathcal{D}_4}{\partial H}\end{bmatrix} \begin{bmatrix}\frac{\partial m}{\partial \eta_1} & \frac{\partial m}{\partial \eta_2} \\ \frac{\partial H}{\partial \eta_1} & \frac{\partial H}{\partial \eta_2} \end{bmatrix} = \begin{bmatrix}\frac{\partial \mathcal{D}_4}{\partial m} & \frac{\partial \mathcal{D}_4}{\partial H}\end{bmatrix} \begin{bmatrix} I & 0 \\ -2m & I \end{bmatrix}
$$

Then, using the update rule and following notation similar to  \citet{hensman2013gaussian}, we get:
\begin{equation}
\begin{split}
    \boldsymbol{\theta}_{2(t+1)} &= \boldsymbol{\theta}_{2(t)} - l\frac{\partial \hat{\mathcal{D}}_4}{\partial \eta_2} \\
    \boldsymbol{\theta}_{2(t+1)} &= \boldsymbol{\theta}_{2(t)} - l\left(\frac{\partial \hat{\mathcal{D}}_4}{\partial \m} \frac{\partial \m}{\partial \eta_2} + \frac{\partial \hat{\mathcal{D}}_4}{\partial H} \frac{\partial H}{\partial \eta_2}\right) \\
    \boldsymbol{\theta}_{2(t+1)} &= \boldsymbol{\theta}_{2(t)} - l\frac{\partial \hat{\mathcal{D}}_4}{\partial H} \\
    -\frac{1}{2} H_{(t+1)}^{-1} &= -\frac{1}{2} H_{(t)}^{-1} - l\frac{\partial \hat{\mathcal{D}}_4}{\partial H} \\
    H_{(t+1)} &= \left(H_{(t)}^{-1} + 2l\frac{\partial \hat{\mathcal{D}}_4}{\partial H}\right)^{-1}
\end{split}
\end{equation}
and
\begin{equation}
\begin{split}
    \boldsymbol{\theta}_{1(t+1)} &= \boldsymbol{\theta}_{1(t)} - l\frac{\partial \hat{\mathcal{D}}_4}{\partial \eta_1} \\
    \boldsymbol{\theta}_{1(t+1)} &= \boldsymbol{\theta}_{1(t)} - l\left(\frac{\partial \hat{\mathcal{D}}_4}{\partial \m} \frac{\partial \m}{\partial \eta_1} + \frac{\partial \hat{\mathcal{D}}_4}{\partial H} \frac{\partial H}{\partial \eta_1}\right) \\
    \boldsymbol{\theta}_{1(t+1)} &= \boldsymbol{\theta}_{1(t)} - l\left(\frac{\partial \hat{\mathcal{D}}_4}{\partial \m} - 2\frac{\partial \hat{\mathcal{D}}_4}{\partial H}\m\right) \\
    H_{(t+1)}^{-1}\m_{(t+1)} &= H_{(t)}^{-1}\m_{(t)} - l\left(\frac{\partial \hat{\mathcal{D}}_4}{\partial \m} - 2\frac{\partial \hat{\mathcal{D}}_4}{\partial H}\m_{(t)}\right) \\
    \m_{(t+1)} &= H_{(t+1)}\left(H_{(t)}^{-1}\m_{(t)} - l\left(\frac{\partial \hat{\mathcal{D}}_4}{\partial \m} - 2\frac{\partial \hat{\mathcal{D}}_4}{\partial H}\m_{(t)}\right)\right)
\end{split}
\end{equation}

\section{Predictive distribution}
The problem of obtaining the predictive distribution for the high-dimensional, out-of-sample data $\y_*$ given covariates $\x_*$ with L-VAE can be split into two parts: obtaining the predictive distribution of the latent representation $\z_*$ and propagating the obtained distribution through the probabilistic decoder $p(y_*|z_*)$. Given the training samples $Y$, covariate information $X$, the learnt parameters of the generative model $\omega = \{\psi, \theta\}$, and the inference model $\phi$, the predictive distribution follows:
\begin{align}
p_\omega(\y_*|\x_*,Y,X) &= \int_{\z_*} p_\omega(\y_*|\z_*,\x_*,Y,X) p_\omega(\z_*|\x_*,Y,X) d\z_* \nonumber \\
&= \int_{\z_*,Z} \underbrace{p_\psi(\y_*|\z_*)}_{\text{decode GP prediction}} \underbrace{p_\theta(\z_*|\x_*,Z,X)}_{\text{GP predictive posterior}} \underbrace{ p_\omega(Z|Y,X)}_{\text{posterior of $Z$}} d\z_* dZ  \nonumber\\
\label{eq:pred_expanded}
&\approx \int_{\z_*,Z} \underbrace{p_\psi(\y_*|\z_*)}_{\text{decode GP prediction}} \underbrace{p_\theta(\z_*|\x_*,Z,X)}_{\text{GP predictive posterior}} \underbrace{q_\phi(Z|Y,X)}_{\text{encode training samples}} d\z_* dZ. 
\end{align}
The true posterior $p_\omega(Z|Y,X)$ is intractable and, similar to the model inference or learning problem, it is replaced with the variational approximation defined by the inference model $q_\phi(Z|Y,X)$. Given such a substitution, the (approximate) predictive GP distribution for the latent representation is available in closed form,
\begin{equation*}
    \hat{p}_\omega(\z_*|\x_*,Y,X) = \int_{Z} p_\theta(\z_*|\x_*,Z,X) q_\phi(Z|Y,X) dZ = N(\z_*|\boldsymbol{\mu}_*, \Sigma_*).
\end{equation*}
Since, in the current work, we only consider multi-output GPs with diagonal cross-covariance functions and the output of probabilistic encoder $p_\psi(\y_n|\z_n)$ is restricted to be a multivariate normal distribution with diagonal covariance matrix, the predictive distribution in the latent space also factorises across the latent dimensions $N(\z_*|\boldsymbol{\mu}_*, \Sigma_*) = \prod_{l=1}^{L} N(z_{*l}|\mu_{*l}, \sigma^2_{*l})$ with:
\begin{align*}
\mu_{*l} &= K_{\x_*X}^{(l)}\Sigma_{l}^{-1} \bar{\boldsymbol{\mu}}_{l} \\
\sigma^2_{*l} &= k_{\x_*\x_*}^{(l)} - K_{\x_*X}^{(l)}\Sigma_l^{-1}K_{X\x_*}^{(l)} + K_{\x_*X}^{(l)}\Sigma_l^{-1}W_l\Sigma_l^{-1}K_{X\x_*}^{(l)} + \sigma_{zl}^2,
\end{align*}
where $k_{\x_*\x_*}^{(l)} = K_{\x_*\x_*}^{(l)} = \sum_{r=1}^R K_{\x_*\x_*}^{(l,r)}(\theta_l^{(r)})$ and $K_{\x_*X}^{(l)} = K_{X\x_*}^{(l)T} = \sum_{r=1}^R K_{\x_*X}^{(l,r)}(\theta_l^{(r)})$ (see eq.~(4) in the main text), and $\Sigma_l = \sum_{r=1}^R K_{XX}^{(l,r)}(\theta_l^{(r)}) + \sigma_{zl}^2 I$,  $W_l = \text{diag}({\sigma}^2_{\phi,l}(\y_1), \ldots, {\sigma}^2_{\phi,l}(\y_N))$ and $\bar{\boldsymbol{\mu}}_{l} = [{\mu}_{\phi,l}(\y_1), \ldots, {\mu}_{\phi,l}(\y_N)]^T$ (as in eq.~(8) in the main text). 
Incorporating this property with eq.~\eqref{eq:pred_expanded} leads to, 
\begin{align*}
p_\omega(\y_*|\x_*,Y,X) &\approx \int_{\z_*} \prod_{d=1}^D \mathcal{N} \left(y_{*d} | g_{\psi,d}(\z_*), \sigma^2_{y d} \right)   \prod_{l=1}^{L} \mathcal{N}(z_{*l}|\mu_{*l}, \sigma^2_{*l}) d\z_*.
\end{align*}

\section{Scalable predictive distribution}
Computing the predictive distribution, as described above, requires performing cubic operations over the $N \times N$ matrices, which makes it practically infeasible for problems with large training data. In this section, we exploit the same inducing points paradigm that was used for $\mathcal{D}_2$  in eq.~\eqref{eq:kld_bound_novel} to alleviate the cubic complexity. For simplicity of notation, we omit the latent dimension index $l$ in the rest of the section. We use the same notation as in eq.~\eqref{eq:elbo_novel} and set $U=K_{SS}^{(A)} +  K_{SX}^{(A)}\hat{\Sigma}^{-1}K_{XS}^{(A)}$. Let $X_* = [\x_{1*}, \ldots, \x_{N'*}]$ denote a collection of $N'$ test data points from $P'$ many subjects. Then,
\begin{align}
\bar{\boldsymbol{\mu}}_{*} &= K_{X_*X}\Sigma^{-1} \bar{\boldsymbol{\mu}} = \left( K_{X_*S}^{(A)}K_{SS}^{(A)^{-1}}K_{SX}^{(A)} + K_{X_*X}^{(R)} \right) \left( K_{XS}^{(A)}K_{SS}^{(A)^{-1}}K_{SX}^{(A)} + \hat{\Sigma} \right)^{-1} \bar{\boldsymbol{\mu}} \nonumber\\ &=\left( K_{X_*S}^{(A)}K_{SS}^{(A)^{-1}}K_{SX}^{(A)} + K_{X_*X}^{(R)} \right) \left( \hat{\Sigma}^{-1}\bar{\boldsymbol{\mu}} - \hat{\Sigma}^{-1} K_{XS}^{(A)}U^{-1}K_{SX}^{(A)} \hat{\Sigma}^{-1}\bar{\boldsymbol{\mu}} \right) 
\end{align}
In practice, the efficient computations are executed in the following order: 
\begin{enumerate}
    \item Compute $\hat{\Sigma}^{-1}\bar{\boldsymbol{\mu}}$; exploiting the block-diagonal property of $\hat{\Sigma}^{-1}$.
    \item Compute $K_{XS}^{(A)}U^{-1}K_{SX}^{(A)} (\hat{\Sigma}^{-1}\bar{\boldsymbol{\mu}})$ using the low-rank properties of all matrices involved.
    \item Compute $\tilde{\boldsymbol{\mu}} = \hat{\Sigma}^{-1}\bar{\boldsymbol{\mu}} - \hat{\Sigma}^{-1} K_{XS}^{(A)}U^{-1}K_{SX}^{(A)} \hat{\Sigma}^{-1}\bar{\boldsymbol{\mu}}$.
    \item Compute $K_{X_*S}^{(A)}K_{SS}^{(A)^{-1}}K_{SX}^{(A)} \tilde{\boldsymbol{\mu}}$ using the low-rank properties of all matrices involved.
    \item Compute $K_{X_*X}^{(R)} \tilde{\boldsymbol{\mu}}$. This relies on the fact that the $K_{X_*X}^{(R)}$ matrix will be very sparse, which can be exploited either via sparse matrix operations or simply by cycling over the sets of rows that correspond to each subject and performing the multiplication only for the non-zero part of $K_{X_*X}^{(R)}$ (if none, this product will be zero for the given subject).
    \item Finally, sum up $\bar{\boldsymbol{\mu}}_{*} = K_{X_*S}^{(A)}K_{SS}^{(A)^{-1}}K_{SX}^{(A)} \tilde{\boldsymbol{\mu}} + K_{X_*X}^{(R)} \tilde{\boldsymbol{\mu}}$
\end{enumerate}

In practice, the predictive distribution will be subject to the inducing point locations $S$ and is guaranteed to be exact if and only if $\u$ is a sufficient statistic of $\f_A$.

The predictive mean can also be expressed in terms of the variational parameters:
$$
K_{X_*S}^{(A)}{K_{SS}^{(A)}}^{-1}\m + K_{X_*X}^{(R)} \hat{\Sigma}^{-1}(\bar{\boldsymbol{\mu}}-K_{XS}^{(A)}{K_{SS}^{(A)}}^{-1}\m)
$$
Also, the predictive covariance can be expressed as follows:
$$
\left(K_{X_*S}^{(A)}-K_{X_*X}^{(R)} \hat{\Sigma}^{-1}K_{XS}^{(A)}\right){K_{SS}^{(A)}}^{-1}H{K_{SS}^{(A)}}^{-1}\left(K_{X_*S}^{(A)}-K_{X_*X}^{(R)} \hat{\Sigma}^{-1}K_{XS}^{(A)}\right)^T + \sigma_z^2 I_{N'} + K_{X_*X_*}^{(R)} - K_{X_*X}^{(R)} \hat{\Sigma}^{-1}K_{XX_*}^{(R)}
$$

\begin{figure}[!b]
  \begin{center}
    \includegraphics[width=0.6\textwidth]{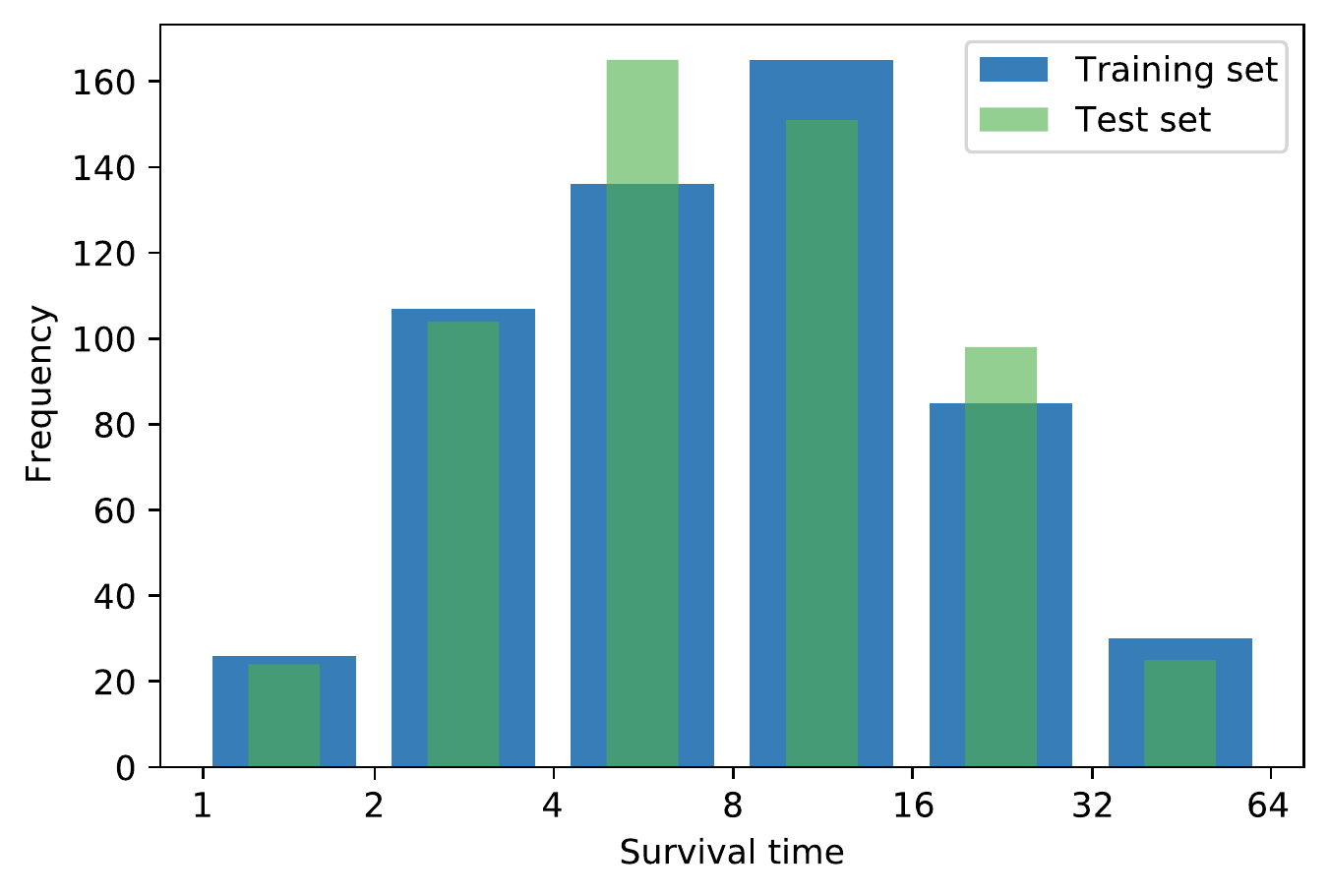}
    \caption{Histogram of survival times in training and test set}
    \label{fig:hist_survival}
  \end{center}
\end{figure}
\section{Downstream classification task in the healthcare data experiment}
Given the test set, our objective is to predict the patient \textit{mortality}. Since our proposed model is generative in nature, to allow such an objective, we shall include the \textit{mortality} covariate in the additive GP prior. The downstream classification task is done by computing the probability that $mortality=0$ and $mortality=1$. 
Given the learnt L-VAE model with fixed parameters $\phi,\psi,\theta$, we can obtain the probability for the binary mortality event for each patient $(X_*,Y_*)$ in the test set as,
{\small
\begin{align*}
    P(\text{mortality}=0) &= \frac{\exp(\mathcal{L}(\phi,\psi,\theta;Y_*,X_*, \text{mortality}=0))}{\sum_{i=0}^1\exp(\mathcal{L}(\phi,\psi,\theta;Y_*,X_*, \text{mortality}=i))},\\
    P(\text{mortality}=1) &= 1 - P(\text{mortality}=0).
\end{align*}
}%
Furthermore, we also introduce a time to mortality covariate (or \textit{mortalityTime}) that is based on the survival time. This covariate is relevant for individuals whose \textit{mortality} = 1 and is treated as missing for the individuals that survive. However, in the testing phase, the exact mortality time is not known as our objective is to perform classification based on $\textit{mortality}$. To overcome this problem, we estimate the distribution of the \textit{mortalityTime} covariate based on the values of the covariate in the training set and, in the test, we compute the expectation of the ELBO described in eq.~(5) of the main manuscript w.r.t.\ to that distribution. 

We approximate the expectation with a finite weighted average. Concretely, we first allocate the \textit{mortalityTime} covariate values in the training set into $B$ bins based on a logarithmic scale. Let $\alpha_i$ be the bin count and $t_i$ be the average value of the \textit{mortalityTime} in the $i$\textsuperscript{th} bin. The proportion of \textit{mortalityTime} values in bin $i$ is $w_i = \alpha_i/(\sum_{i=1}^B \alpha_i)$. The weighted ELBO is then computed as:
\begin{align*}
    \mathcal{L}(\phi,\psi,\theta;Y,X, \text{mortality}=1) = \sum_{i=1}^B w_i \cdot \mathcal{L}(\phi,\psi,\theta;Y,X, \text{mortality}=1, t_i),
\end{align*}
where the ELBO terms are now explicitly conditioned by the \textit{mortalityTime}, $t_i$. We use $B=6$ in our analysis. The histograms of the survival times in the training and test data are shown in Fig.~\ref{fig:hist_survival}. 

We follow the data preprocessing steps described in \citet{luo2018multivariate}, and standardise the measurements of the 35 different attributes.

\section{Optimisation and practical considerations}
We make use of a suitable stochastic optimisation technique to minimise the ELBO in eq.~(4) of the main manuscript. The parameters that we need to optimise include the neural network weights ($\phi$ and $\psi$) and kernel parameters ($\theta$) of the multi-output additive GPs. In particular, the optimisation is done using the Adam optimiser \citep{kingma2014adam}, which is an adaptive learning rate method that maintains an exponentially decaying average of past gradients as well as squared gradients. In case of mini-batch training, the Adam steps are conducted interchangeably with natural gradient-based updates of the variational parameters.  For the inference implementation, we make use of PyTorch \citep{paszke2019pytorch} which allows the computation of derivatives using automatic differentiation.

In all the experiments, we first pre-train the neural networks with a standard normal distribution as the prior on the latent space (standard VAE \citep{kingma2013auto}) for $1000$ epochs. This is followed by training the L-VAE model using the pre-trained encoder and decoder networks as initial values for $\phi$ and $\psi$, respectively. While training the L-VAE model, we monitor the loss on the validation (independent) dataset as a performance metric. Similar to the strategy of early stopping, we save the weights of the model that has the best performance on our defined metric. These model weights are chosen to perform predictions and other downstream tasks. However, to handle the possibility of a local minimum, we do not specify a stopping criterion and continue the training procedure till a predefined number of epochs has been performed. In the Rotated MNIST and Health MNIST experiments, L-VAE is trained for a maximum of $2000$ epochs. Moreover, in the Healthcare data experiment, L-VAE is trained for a maximum of $1000$ epochs.

\section{Supplementary tables}
Table~\ref{table:nnet_spec_rotatedMNIST} describes the neural network architecture used for the Rotated MNIST experiment. The hyperparameter choices are similar to \citet{casale2018gaussian}. The architecture used for the Health MNIST experiment is described in table~\ref{table:nnet_spec_healthMNIST}. We have tried to replicate the hyperparameter choices from \citet{fortuin2019multivariate} for this experiment. For the Physionet Challenge 2012 dataset, we did not make use of a convolutional neural network (CNN) as was done for the Rotated MNIST as well as Health MNIST experiments because CNNs are more appropriate for image based (visual) data where the regional correlation (receptive field) of the measured values is important \citep{goodfellow2016deep}. Table~\ref{table:nnet_spec_physionet} describes the architecture for the multi layered perceptron (MLP) that we used for the Physionet Challenge 2012 dataset. It is similar to the architecture used in \citet{fortuin2019multivariate}.
\begin{table}[!h]
\begin{center}
\begin{adjustbox}{width=0.8\textwidth}
\begin{tabular}{ c l r } 
\hline
 & Hyperparameter & Value \\
\hline
\multirow{9}{6em}{Inference network} & Dimensionality of input & $28 \times 28$ \\ 
& Number of convolution layers & 3 \\ 
& Number of filters per convolution layer & 72 \\ 
& Kernel size &  $3 \times 3$\\
& Stride &  2\\
& Number of feedforward layers & 1 \\
& Width of feedforward layers & 128 \\
& Dimensionality of latent space & $L$\\
& Activation function of layers & ELU \\
\hline
\multirow{8}{6em}{Generative network} & Dimensionality of input & $L$ \\
& Number of transposed convolution layers & 3 \\
& Number of filters per transposed convolution layer & 72 \\
& Kernel size &  $3 \times 3$\\
& Stride &  1\\
& Number of feedforward layers & 1 \\
& Width of feedforward layers & 128 \\
& Activation function of layers & ELU \\
\hline
\end{tabular}
\end{adjustbox}
\end{center}
\caption{Neural network architectures used in the Rotated MNIST dataset.}
\label{table:nnet_spec_rotatedMNIST}
\end{table}

\begin{table}[!h]
\begin{center}
\begin{adjustbox}{width=0.8\textwidth}
\begin{tabular}{ c l r } 
\hline
 & Hyperparameter & Value \\
\hline
\multirow{12}{6em}{Inference network} & Dimensionality of input & $36 \times 36$ \\ 
& Number of convolution layers & 2 \\ 
& Number of filters per convolution layer & 144 \\ 
& Kernel size &  $3 \times 3$\\
& Stride &  2\\
& Pooling & Max pooling\\
& Pooling kernel size & $2 \times 2$\\
& Pooling stride & 2\\
& Number of feedforward layers & 2 \\
& Width of feedforward layers & 300, 30 \\
& Dimensionality of latent space & $L$\\
& Activation function of layers & RELU \\
\hline
\multirow{8}{6em}{Generative network} & Dimensionality of input & $L$ \\
& Number of transposed convolution layers & 2 \\
& Number of filters per transposed convolution layer & 256 \\
& Kernel size &  $4 \times 4$\\
& Stride &  2\\
& Number of feedforward layers & 2 \\
& Width of feedforward layers & 30, 300 \\
& Activation function of layers & RELU \\
\hline
\end{tabular}
\end{adjustbox}
\end{center}
\caption{Neural network architectures used in the Health MNIST dataset.}
\label{table:nnet_spec_healthMNIST}
\end{table}

\begin{table}[!h]
\begin{center}
\begin{adjustbox}{width=0.8\textwidth}
\begin{tabular}{ c l r } 
\hline
 & Hyperparameter & Value \\
\hline
\multirow{5}{6em}{Inference network} & Dimensionality of input & 35 \\ 
& Number of feedforward layers & 2 \\
& Number of elements in each feedforward layer & 128, 64 \\
& Dimensionality of latent space & $L$\\
& Activation function of layers & RELU \\
\hline
\multirow{4}{6em}{Generative network} & Dimensionality of input & $L$ \\
& Number of feedforward layers & 2 \\
& Number of elements in each feedforward layer & 64, 128 \\
& Activation function of layers & RELU \\
\hline
\end{tabular}
\end{adjustbox}
\end{center}
\caption{Neural network architectures used in the Physionet Challenge 2012  dataset.}
\label{table:nnet_spec_physionet}
\end{table}

\FloatBarrier

\section{Supplementary images}
\begin{figure}[!ht]
  \begin{center}
    \includegraphics[width=0.8\linewidth]{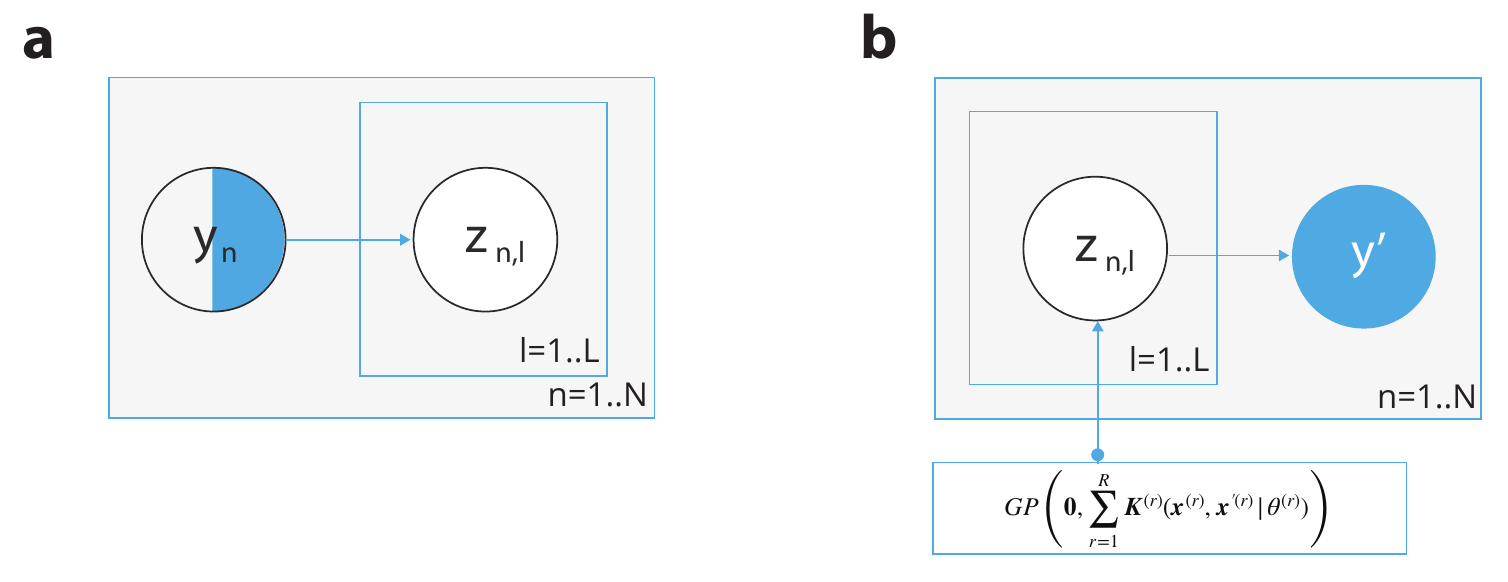}
    \caption{Plate diagram of the model. The shaded circle refers to an observed variable, the partially shaded circle refers to a partially observed variable (due to missing values), and the un-shaded circle refers to an un-observed variable. \textbf{(a)} Represents the inference (or encoder) model and \textbf{(b)} Represents the generative (or decoder) model.}
    \label{fig:plate_diagram}
  \end{center}
\end{figure}

\begin{figure}[!h]
  \begin{center}
    \includegraphics[width=0.8\linewidth]{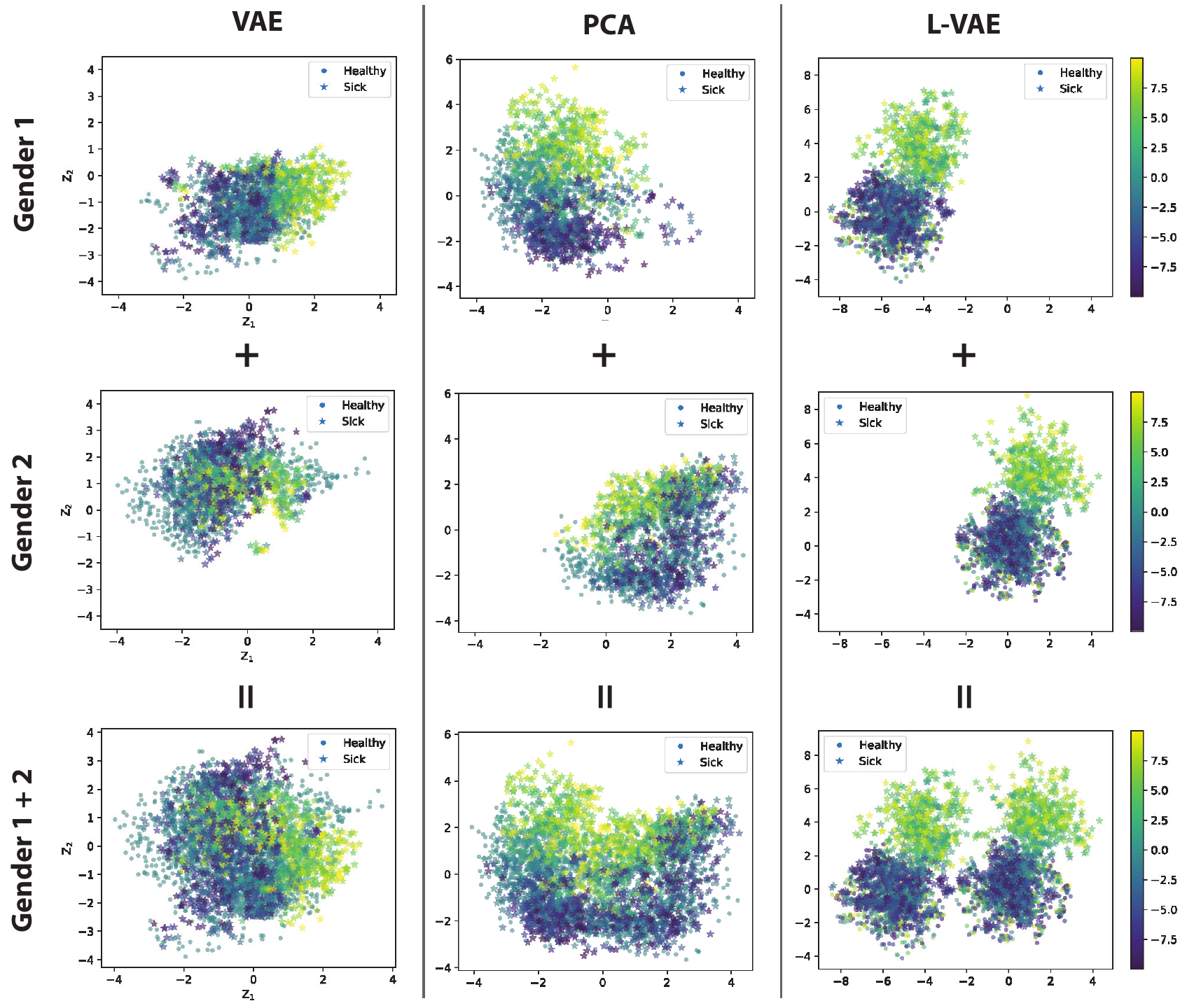}
    \caption{Comparison of the resulting latent space using VAE, PCA, and L-VAE on the Health MNIST dataset. The L-VAE model is fit using $\f_{\mathrm{ca}}(\mathrm{id}) + \f_{\mathrm{se}}(\mathrm{age}) + \f_{\mathrm{ca} \times \mathrm{se}}(\mathrm{id} \times \mathrm{age}) + \f_{\mathrm{ca} \times \mathrm{se}}(\mathrm{sex} \times \mathrm{age})+ \f_{\mathrm{ca} \times \mathrm{se}}(\mathrm{diseasePresence} \times \mathrm{diseaseAge})$ as the multi-output additive GP prior. The number of latent dimensions is set to 2. The points are coloured according to the $\mathrm{diseaseAge}$ as shown in the colour bar.}
    \label{fig:latent_comparison}
  \end{center}
\end{figure}

\begin{figure}[!h]
  \begin{center}
    \includegraphics[width=1\linewidth]{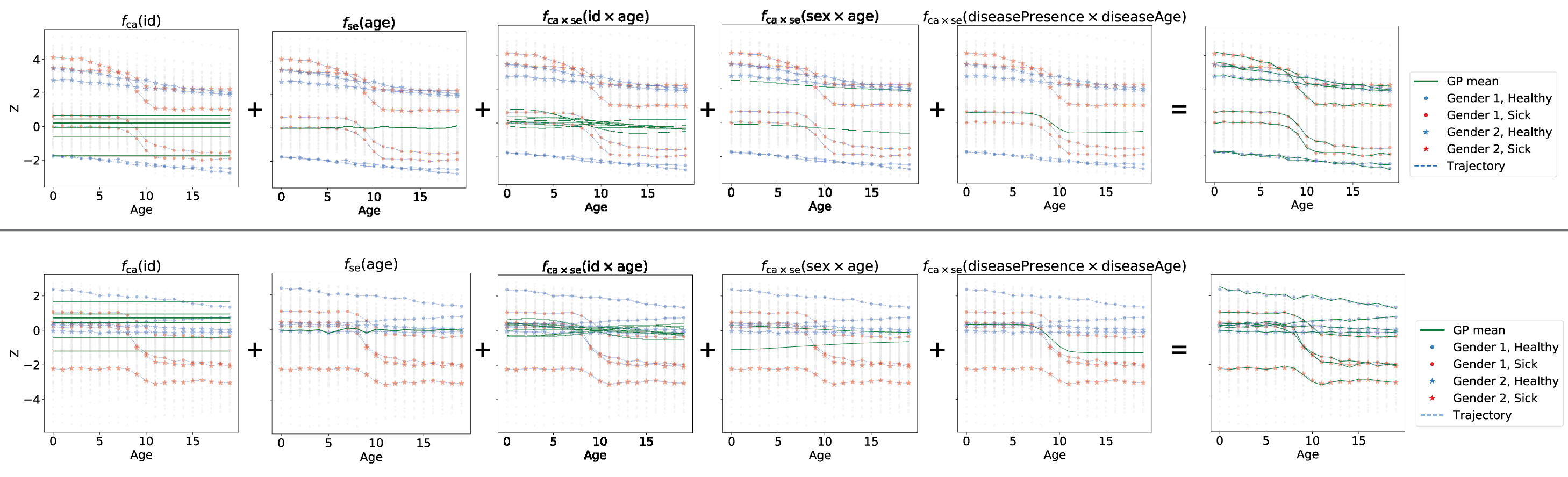}
    \caption{GP model fittings of L-VAE in the latent space with dimension 2 on the Health MNIST dataset.}
    \label{fig:GP_fitting}
  \end{center}
\end{figure}

\begin{figure}[!h]
  \begin{center}
    \includegraphics[width=0.65\linewidth]{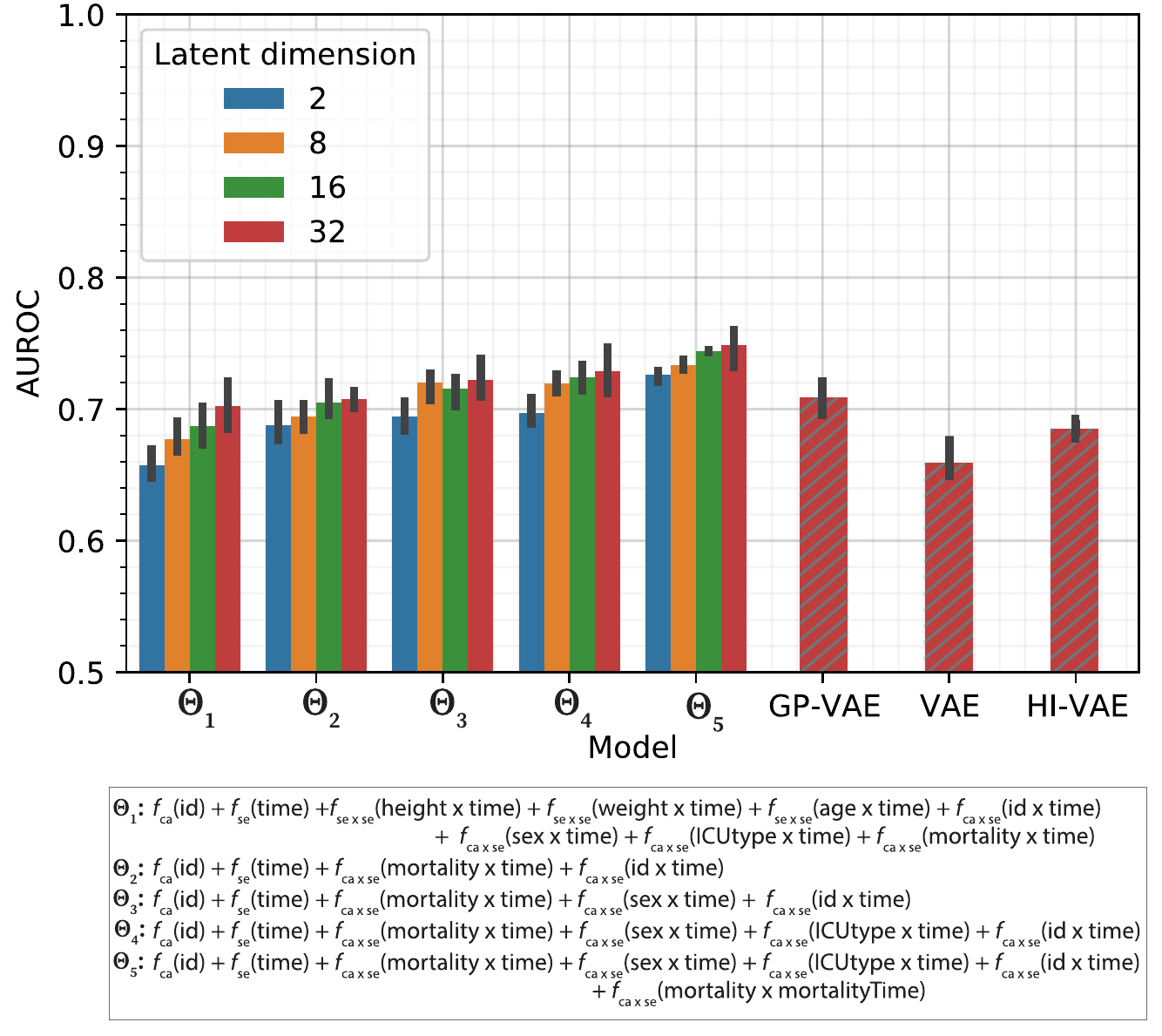}
    \caption{AUROC scores for the patient mortality prediction task on the test set of the Physionet Challenge 2012 dataset. This is an extension to Fig.~4 in the main manuscript. In this figure, we can observe L-VAE's performance with different latent dimensions as well as all the patient-specific general auxiliary covariates. Higher AUROC score is better.}
    \label{fig:physionet_ful}
  \end{center}
\end{figure}

\FloatBarrier
\bibliographystyle{abbrvnat}
\bibliography{references}